%% file: example_paper.tex
\documentclass{article}

\usepackage{microtype}
\usepackage{graphicx}
\usepackage{booktabs} 
\usepackage{tikz}
\usepackage{multirow}
\usepackage{subcaption}
\usepackage[table,xcdraw]{xcolor}
\usepackage{svg}

\usepackage{hyperref}

\usepackage[preprint]{icml2026}

\usepackage{amsmath}
\usepackage{amssymb}
\usepackage{mathtools}
\usepackage{amsthm}

\usepackage[capitalize,noabbrev]{cleveref}

\theoremstyle{plain}
\newtheorem{theorem}{Theorem}[section]
\newtheorem{proposition}[theorem]{Proposition}
\newtheorem{lemma}[theorem]{Lemma}
\newtheorem{corollary}[theorem]{Corollary}
\theoremstyle{definition}

\theoremstyle{remark}
\newtheorem{remark}[theorem]{Remark}

\usepackage[textsize=tiny]{todonotes}

\input{sections/htcl_macros}

\begin{document}

\twocolumn[
\icmltitle{Mitigating Task-Order Sensitivity and Forgetting via Hierarchical Second-Order Consolidation}



\icmlsetsymbol{equal}{*}

\begin{icmlauthorlist}
\icmlauthor{Protik Nag}{uofsc}
\icmlauthor{Krishnan Raghavan}{anl}
\icmlauthor{Vignesh Narayanan}{uofsc}
\end{icmlauthorlist}

\icmlaffiliation{uofsc}{Department of Computer Science and Engineering, University of South Carolina, SC, USA}
\icmlaffiliation{anl}{Argonne National Laboratory, Illinois, USA}

\icmlcorrespondingauthor{Protik Nag}{pnag@email.sc.edu}

\icmlkeywords{Machine Learning, ICML}

\vskip 0.3in
]



\printAffiliationsAndNotice{}  

\begin{abstract}
We introduce \textbf{Hierarchical Taylor Series-based Continual Learning (HTCL)}, a framework that couples fast local adaptation with conservative, second-order global consolidation to address the high variance introduced by random task ordering.  To address task-order effects, HTCL identifies the best intra-group task sequence and integrates the resulting local updates through a Hessian-regularized Taylor expansion, yielding a consolidation step with theoretical guarantees. The approach naturally extends to an $L$-level hierarchy, enabling multiscale knowledge integration in a manner not supported by conventional single-level CL systems. Across a wide range of datasets and replay and regularization baselines, HTCL acts as a model-agnostic consolidation layer that consistently enhances performance, yielding mean accuracy gains of $7\%$ to $25\%$ while reducing the standard deviation of final accuracy by up to $68\%$ across random task permutations.

\end{abstract}

\input{sections/introduction}
\input{sections/short-related-work}
\input{sections/method}

\input{sections/experiments}
\input{sections/conclusion} 

\section*{Acknowledgements}
Authors VN and PN would like to acknowledge funding support from the  Air Force Office of Scientific Research under grant FA9550-24-1-0228 and the National Science Foundation under the grant 2337998. KR, was supported by the U.S. Department of Energy, Office of Science (SC), Advanced Scientific Computing Research (ASCR), Competitive Portfolios Project on Energy Efficient Computing: A Holistic Methodology, under Contract DE-AC02-06CH11357. We also acknowledge the support by the U.S. Department of Energy for the SciDAC 6 RAPIDS institute.

We have utilized generative AI, particularly, claude-code from Anthropic to generate the plots, tables in the experiment. Some of the graphics and text in the appendix has been generated, \emph{claude-code}. Rest of the paper is the original product of the authors.

\section*{Impact Statement}


This paper presents work whose goal is to advance the field of Machine Learning. There are many potential societal consequences of our work, none which we feel must be specifically highlighted here.



\bibliography{example_paper}
\bibliographystyle{icml2025}

\newpage
\input{sections/appendix_v2}


\end{document}

%% file: sections/htcl_macros.tex
\newcommand{\wh}[1]{\mathbf{w}_{#1}}            
\newcommand{\wl}{\mathbf{w}_{\ell}}             
\newcommand{\w}{\mathbf{w}}                     
\newcommand{\wopt}{\mathbf{w}^{*}}              
\newcommand{\whT}[2]{\mathbf{w}_{#1}^{(#2)}}    
\newcommand{\wlT}[1]{\wl^{(#1)}}                

\newcommand{\T}{\mathbb{T}}                    
\newcommand{\tsk}{\tau}                         
\newcommand{\G}{\mathrm{G}}                    
\newcommand{\D}{\mathrm{D}}                    

\newcommand{\perm}{\pi(t)}                         
\newcommand{\permset}{\Pi(t)}                      
\newcommand{\gpermset}{\varsigma}               
\newcommand{\gperm}{\sigma}                     

\newcommand{\loss}{\mathrm{L}}                 
\newcommand{\J}{\mathcal{J}}                    

\newcommand{\reg}{\lambda}                      
\newcommand{\Dw}{\Delta\mathbf{w}}              
\newcommand{\Dd}{\Delta\mathbf{d}}              

\newcommand{\Acc}{A}                            
\newcommand{\Eval}{\mathrm{Eval}}               
\newcommand{\TrainSeq}{\mathrm{TrainSeq}}       

\newcommand{\nlev}{L}                           

\newcommand{\ntasks}{t}                         
\newcommand{\ngroups}{m}                        
\newcommand{\gsize}{k}                          
\newcommand{\pdim}{d}                           
\newcommand{\ncatch}{N_{\mathrm{catch}}}        

\newcommand{\E}{\mathbb{E}}                     
\newcommand{\R}{\mathbb{R}}                     

\DeclareMathOperator*{\argmax}{arg\,max}


%% file: sections/introduction.tex
\section{Introduction}\label{sec:intro}
Continual learning (CL) aims to enable neural networks to acquire knowledge sequentially without catastrophically forgetting previously learned information~\cite{kirkpatrick2017overcoming, parisi2019continual}. This capability, essential for real-world applications, has been studied extensively through replay~\cite{rebuffi2017icarl}, regularization~\cite{zenke2017continual}, and architecture-based approaches~\cite{fernando2017pathnet}. Despite impressive progress in CL, a fundamental challenge remains: \emph{task-order sensitivity}.

\begin{figure}[!htbp]
    \centering
    \includegraphics[width=\linewidth]{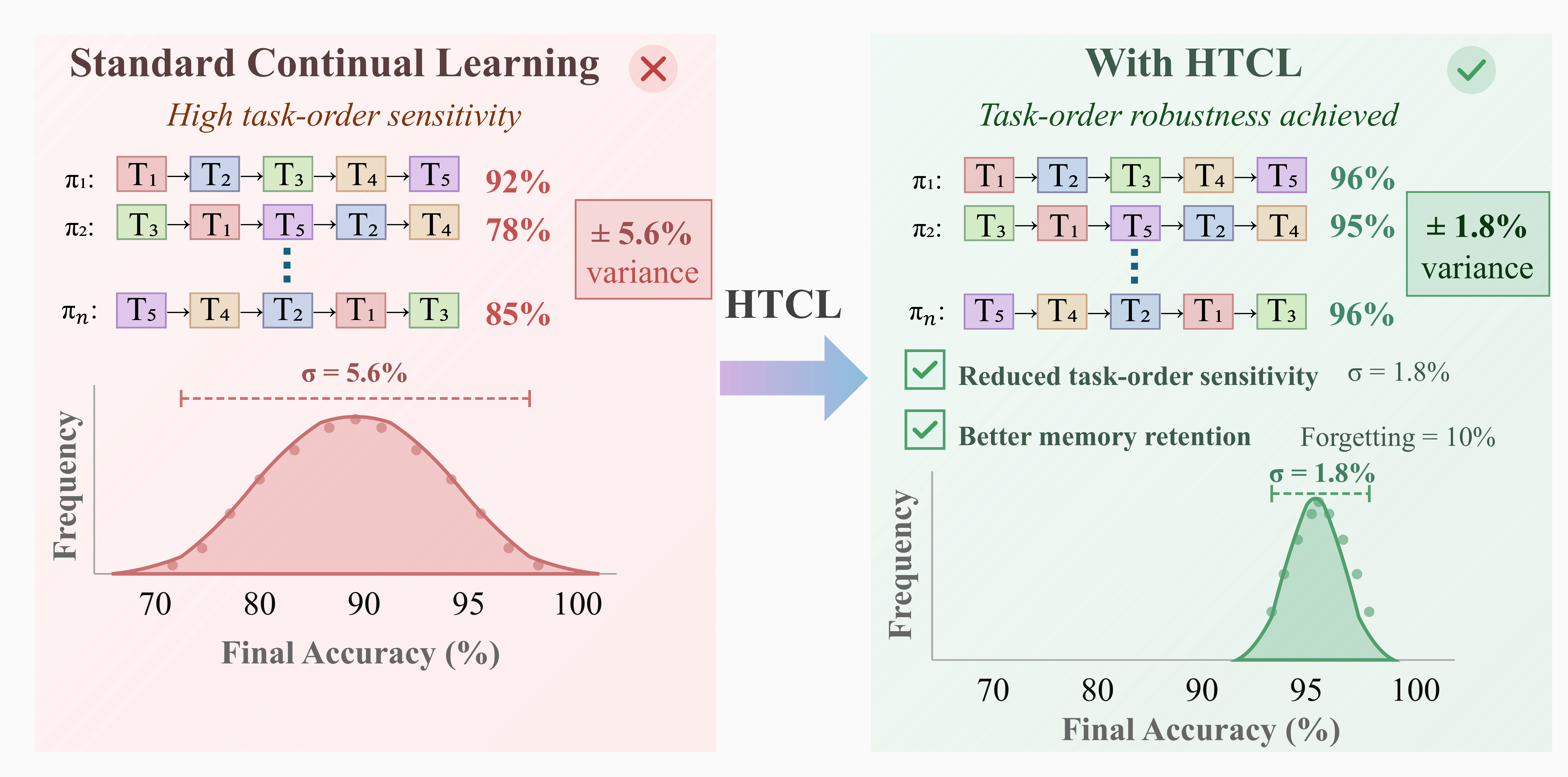}
    \caption{Task-order sensitivity in CL and its mitigation via HTCL. \textbf{Left:} Standard CL methods exhibit high sensitivity to task ordering. \textbf{Right:} HTCL achieves consistent performance regardless of task ordering, producing a narrow accuracy distribution.}
    \label{fig:new-cl-goal}
    \vspace{-6mm}
\end{figure}
In real-world deployments, task sequences are rarely under the practitioner's control~\cite{de2021continual}. Unknown or random task orderings are common, and in such settings, 
the existing CL methods demonstrate increased variance in performance~(see experiments in Appendix~\ref{app:impact-example} where task ordering induces up to $22\%$ variance), consistent with prior findings~\cite{bell2022effect, riemer2018learning}. Notably, even state-of-the-art approaches such as strong experience replay (SER)~\cite{zhuo2023continual} show variance of $5.6^2$ across 120 task permutations indicating that memory retention is unstable. Performance depends on incidental ordering effects rather than the tasks themselves, which is a major concern. For instance, a medical diagnosis model, learning to identify new diseases, cannot  depend solely on the order in which diseases are presented to the model.  


A naive solution is to evaluate all possible task sequences and select the best one. However, for $n$ tasks this requires evaluating $n!$ permutations, over 3.6 million for $n=10$, and approximately $2.4 \times 10^{18}$ for $n=20$, rendering exhaustive search computationally intractable. This NP-hardness creates a fundamental tension: task ordering profoundly affects performance, yet optimizing over all possible orderings is infeasible.

To address the task ordering problem, we introduce \textbf{Hierarchical Taylor Series-based Continual Learning (HTCL)}, a multi-model framework that couples a fast-adapting local model with a conservative, second-order global consolidation mechanism to minimize the impact of task ordering on the final performance of the model tractabily. Our key contributions are three fold.

\textbf{Tractable order optimization via grouping:} Rather than attempting to solve the intractable $n!$ permutation problem, HTCL identifies a principled middle ground. Tasks are partitioned into small groups of size $k$, and all $k!$ intra-group orderings are evaluated to select the sequence that yields the best local adaptation that provably improves expected performance over arbitrary or random orderings (Theorem~\ref{prop:group_vs_random}). 

 \textbf{Second-order consolidation for memory retention:} To integrate updates across group's that may represent different orderings, we introduce a Hessian-regularized Taylor expansion (Eq.~\ref{eq:hier_update}) that  approximates ``how the new groups optimal ordering affects the other groups that have already been learned?" We employ scalable low-rank curvature approximations, yielding near-linear time and memory overhead~(Appendix~\ref{app:complexity}) for the second order implementation while maintaining generalization of performance.

 \textbf{Multi-level hierarchy for long task horizons:} To capture the task ordering effects across long horizons, we introduce a multi-level hierarchy where each level operates at a progressively slower timescale and consolidates knowledge across task groups in increasing granularity~(see Figure~\ref{fig:new-cl-goal} and \ref{fig:htcl-overview}). Lower levels retain plasticity for recent tasks while higher levels accumulate stable, long-term knowledge.

Finally, the HTCL is \textbf{model-agnostic}. The local learner can employ any existing CL method. The hierarchical approach enhances stability without altering the underlying algorithm and requires no assumption on the group-wise learning.

\textbf{What to expect empirically?} Our experiments demonstrate two consistent benefits across datasets. First, HTCL reduces task-order sensitivity: two-level HTCL~(HTCL-L2) reduces standard deviation by $33\text{--}68\%$ on SplitMNIST and $17\text{--}21\%$ on CIFAR-100, with SER + HTCL-L2 achieving a $67.86\%$ variance reduction (see Fig.~\ref{fig:ser-results}), these benefits extend to other modalities, yielding variance reductions of approximately $33\textbf{--}38\%$ on graph (CORA) and text (20 Newsgroups) classification tasks. Second, HTCL improves memory retention: mean forgetting drops by up to $70.9\%$ (DER + HTCL-L2 on SplitMNIST) per-task standard deviation reduces by over $30\%$ in longer task sequence (see Fig.~\ref{fig:extended} for CIFAR-100 dataset). 

%% file: sections/short-related-work.tex
\section{Related Works}
\label{sec:short-related-works}
Here we briefly review some of the related works and position our HTCL framework within the CL literature. A detailed review of related works is given in Appendix \ref{app:related_work}. 

Existing methods in CL can be broadly categorized into \emph{regularization-based}~\cite{kirkpatrick2017overcoming, zenke2017continual}, \emph{replay-based}~\cite{zhuo2023continual, buzzega2020dark, shin2017continual}, and \emph{architecture-based} methods~\cite{parisi2019continual, nguyen2019toward}. These methods predominantly approach the CL problem through inductive biases or auxiliary mechanisms with varying memory–computation trade-offs, while remaining agnostic to task ordering.

\textbf{Task Ordering and Sequence Sensitivity.} Recent studies demonstrate that task order profoundly affects CL performance, where identical tasks presented in different sequences yield widely varying accuracies~\cite{bell2022effect, riemer2018learning}.  Addressing this via brute-force optimization is NP-hard due to the factorial complexity of possible permutations~\cite{knoblauch2020optimal}. 

\textbf{Multi-Timescale and Dual-Model Approaches} such as DualNet~\cite{pham2021dualnet} has been proposed. For instance, DualNet employs two interacting fast and slow learners operating at different temporal scales with distinct targets. Unlike such dual-model systems, HTCL employs a consolidation mechanism derived from a second-order Taylor series expansion. Particularly, HTCL approximates the curvature of the loss landscape aimed at reducing the learner's sensitivity to task ordering.

\textbf{Summary and Distinction.} In summary, while prior methods mitigate forgetting, they remain vulnerable to task-ordering. HTCL introduces a principled, multi-level framework that unifies local adaptability and hierarchical stability, yielding a tractable approximation of order invariance that standard CL methods lack. 

%% file: sections/method.tex
\section{Proposed Method}
\label{sec:method}


\subsection{Preliminaries: The Traditional CL Objective}
We define $t=1,2,\ldots,$ to be the instance at which a task $\tau(\cdot)$ is observed and define a set of tasks $\T(t) = \{\tsk(1), \tsk(2), \ldots, \tsk(t)\}$. In a specific experimental setup, these tasks arrive according to a permutation (or sequence) $\perm$. We define $\perm$ as a sequence of task indices, where $\perm[i]$ denotes the index of the task encountered at the $i$-th time step. Consequently, the dataset observed at step $i$ is denoted as $\D_{\perm[i]}$. In a real world scenario, one typically utilizes a CL learner consisting of a parametric model (a neural network for the purposes of this paper) parameterized by $\w \in \R^\pdim$ to learn these set of tasks such that, when for every new task arrives $\tau(t)$, the learner does not forget previously learned tasks while learning to perform the new task. 

\textbf{The classical goal of a CL learner} is to train a single set of parameters $\w$ given $\perm$ using a cumulative loss~(also called forgetting loss in CL) $\J$ to achieve the following objective 
\begin{align}
    \label{eq:standCL}
   min_{\w} \left\{ \J_{\perm}{(\w)} := \sum_{i=1}^{t} \loss(\w, \D_{\perm[i]}) \right\},
\end{align}
where $\loss$ is the loss function for a specific task. After repeated updates, the learner seeks to attain $\wopt$ that minimizes $\J_{\perm}(\w)$. If the task order is shuffled, with $\pi(t)^{(1)} \ne \pi(t)^{(2)}$, the learner may yield a completely different solution 
to \eqref{eq:standCL}. 
Importantly, this will result in significant discrepancies between $J_{\pi(t)^{(1)}}$ and $J_{\pi(t)^{(2)}}$. This leads to a phenomenon we describe as \textbf{``the curse of task ordering"}. Empirical evidence supports this observation (see Appendix~\ref{app: task-order-impact}): for two distinct permutations $\perm^{(1)}$ and $\perm^{(2)}$, the final performance can vary drastically.

\subsection{Problem Formulation: A New Objective}
\label{sec:problem_formulation}

In the presence of random/permutation task orders, it is desirable that $\wopt$ is order/permutation invariant. Let $\permset$ denote the set of all permutations of the $t$ tasks that are available. Each element $\perm^{(i)} \in \permset$ is a complete ordering of $t$ tasks, represents a sequence in which every task appears exactly once. Thus, $\permset = \{\perm^{(1)}, \perm^{(2)}, \ldots, \perm^{(t!)}\}$. Mathematically, we want to minimize the \textbf{expected value} of the loss over all possible task permutations
\begin{align}
    \label{eq:tiCL}
    \J^{*}(\wopt) = \min_{\w} \, \E_{\perm \sim \permset}\!\left[\J_{\perm}(\w)\right].
\end{align}
Particularly, we seek to solve \eqref{eq:tiCL} and construct an algorithm that yields stable and higher performance regardless of the task order. 

\textbf{Intractable combinatorial challenge:}  However, this goal leads to a combinatorial search space. For instance, if $t=20$ tasks, the size of $\permset$ becomes $20! \approx 2.4 \times 10^{18}$. Moreover, as evaluating each sequence requires full training, the compute requirement is significant.

\textbf{Tractable optimization of \eqref{eq:tiCL} via a task grouping strategy.}  To avoid evaluating the full factorial space, which entails over $3.6$ million permutations for $t=10$, we partition the task sequence into smaller groups of size $\gsize.$ For instance, a sequence of $10$ tasks can be decomposed into two groups of $3$ and one group of
$4$. Consequently, rather than evaluating $10!$ sequences, we only evaluate $3! + 3! + 4! = 36$ permutations. This complexity can be further minimized by using groups of $\gsize=2$, reducing the computational cost to just $5 \times 2! = 10$ evaluations. More formally, we partition the $t$ task sequence $\{1, 2, 3, \cdots, t\}$ into $\ngroups$ disjoint groups $\G_1, \dots, \G_\ngroups$, each containing  data corresponding to $\gsize$ tasks (where $t \approx \ngroups \cdot \gsize$). If $t$ is not divisible by $k$ then the last group contains the remaining tasks. 


The grouping strategy fundamentally changes how task ordering affects the final model's performance. It induces greater task-ordering robustness to the performance. Consider $\ntasks = 6$ tasks partitioned into two groups of size $k = 3.$ Without grouping, the $6! = 720$ possible task orderings could each yield a different final model. With HTCL (for the $k=3$ case), all $3! = 6$ arrival orderings~(e.g., $(\tsk_1, \tsk_2, \tsk_3)$, $(\tsk_2, \tsk_1, \tsk_3)$, etc.) are mapped to the same outcome and the arrival order becomes irrelevant within the group. In other words, any two full-sequence permutations that assign the same tasks to the same groups, and differ only in their intra-group arrangement thus reducing the effective number of distinct outcomes from $\ntasks!$ to at most $\frac{\ntasks!}{(k!)^m}$ which reduces the complexity from $\mathcal{O}(\ntasks!)$ to $\mathcal{O}(\ngroups \cdot \gsize!)$. Figure~\ref{fig:new-cl-goal} gives us an overview of how shifting to this new formulation reduces the variance. The theoretical lower bound for the performance improvement due to this group size are established in Theorem~\ref{prop:group_vs_random} (see Appendix~\ref{app:performance-bounds}).

The performance bounds of HTCL in Theorem~\ref{prop:group_vs_random} guarantees that HTCL's per-group selection meets or exceeds the expected accuracy under a uniformly random choice of intra-group orderings. Further, within each group, the CL problem can be solved independently by approximating the expected cost in \eqref{eq:tiCL} and solving the problem with any known CL approach to obtain $\wl$. However, solving for local groups alone is insufficient; recovering the order-invariant solution for \eqref{eq:tiCL} requires effectively integrating these disjoint results across groups. We therefore introduce a hierarchical CL strategy that combines the local model weights $\wl$ obtained from each group to estimate the final optimal parameters $\w$. This aggregation is derived via a Taylor series approximation, as described next.

\subsection{The HTCL Framework: Decoupling Plasticity and Stability}
\label{sec:htcl_framework}
Intuitively, we solve \eqref{eq:tiCL} within each group in an order-agnostic way and then consolidate group-level solutions to obtain approximate invariance to the full task sequence. While this procedure does not yield a globally optimal solution, we report that existing CL methods when used within the HTCL framework leads to significant performance improvements. However, using a single learner to both find the exhaustive solutions and consolidate creates a challenge.

Consider this problem: within each group, we can achieve order-invariance locally by evaluating all $k!$ intra-group permutations and selecting the best one. However, when a new group arrives, directly updating the model increases the chance of disrupting the order-invariant solution for previous groups. Formally, the solution for the new group $\G_{t}$ must not increase the order dependent sensitivity of the order-invariant solutions for groups $\G_1, \cdots, \G_{t-1}$.

Our hierarchical solution addresses this challenge by separating the role of exploration and preservation. HTCL decouples the learning process into multiple components organized in a hierarchy. At the base of this hierarchy sits a \emph{local model} with parameters $\wl \in \R^\pdim$, which is highly plastic and rapidly adapts to the current task group. Above the local model, we maintain one or more \emph{hierarchical models} $\wh{1}, \wh{2}, \ldots, \wh{\nlev}$. The hierarchical model $\wh{1}$ directly integrates knowledge from the local model $\wl$, while $\wh{2}$ integrates from $\wh{1}$, and so forth. The final hierarchical model $\wh{\nlev}$ serves as the \emph{ultimate learner} that stores the most consolidated long-term knowledge across groups and produces the final output of our method. 


We will formally describe this notion in the simplest two-model configuration ($\nlev = 1$) where we have one local model $\wl$ and one hierarchical model $\wh{1}$. At each timestep, we explore one single group, which can be denoted by task index $t$ (since the group is described using the $t^{th}$ task, we overload this notation). The local model $\wlT{t}$ aggressively explores the best task ordering within the current group, while the hierarchical model $\whT{1}{t}$ conservatively integrates the local model's findings with its previous state $\whT{1}{t-1}$, which was the result of the best task ordering from the previous $t-1$ groups. Such separation allows aggressive optimization of task order without risking the long-term memory~(increased sensitivity of the $\whT{1}{t}$ across different orderings from different groups). 

Notably, to capture effects across more number of groups it is easier to introduce more steps in the hierarchy, thus capturing longer horizon dependencies between groups. The local model can be trained using any available CL method.

\subsection{Second Order Hierarchical Update}
\label{sec:hierarchical_update}
With the optimal local model $\w_l^{(t)}$ now identified, we must consolidate its knowledge into the hierarchical model $\whT{1}{t-1}$. A naive approach, such as simple averaging or direct replacement, treats all parameter directions equally. However, not all directions affect order sensitivity equally. Some parameter directions are comparatively \emph{insensitive} and therefore the model can move freely along them without affecting performance on previously consolidated groups. Other directions are \emph{sensitive} to even small perturbations along these directions which cause the loss on prior groups to increase sharply. Therefore, a naive approach will effectively undo the order-invariant solutions we have carefully constructed within a group. To address this issue, this section develops a principled update rule that respects the curvature of the loss surface. We want to move the hierarchical model toward the local model to incorporate newly learned knowledge while maintaining fidelity to the prior groups.

\textbf{Setup and notation.} Let $\J(\w_1)$ denote the cumulative loss using the hierarchical model $\w_1^{(t)}$ over all tasks observed up to time $t$, as defined in Eq.~\ref{eq:standCL}. We denote the gradient and Hessian of this cumulative loss evaluated at the current weights $\mathbf{g}^{(t)} \coloneqq \nabla \J(\whT{1}{t}),
     \mathbf{H}^{(t)} \coloneqq \nabla^2 \J(\whT{1}{t}).$ In practice, we approximate these quantities using the replay buffer and the data of tasks within the current group.

\textbf{Formulating the update as an optimization problem.} We seek an update $\Dw = \whT{1}{t} - \whT{1}{t-1}$ that integrates the new group's order-invariant solution $\wlT{t}$ into the hierarchical model without disrupting the consolidation knowledge from previous groups. The Taylor expansion of $\J$ around $\whT{1}{t-1}$ provides exactly this. It approximates how incorporating $\Dw$ affects the cumulative loss over all previously learned tasks, serving as a proxy for the sensitivity of $\whT{1}{t-1}$ to the introduction of the new group. 

The Taylor expansion of $\J$ around the previous hierarchical weights $\whT{1}{t-1}$ is given as
\begin{align}
    \label{eq:taylor-exp}
    \begin{alignedat}{2}
        \J(\whT{1}{t-1} + \Dw) \approx \J(\whT{1}{t-1}) + {\mathbf{g}^{(t-1)}}^\top \Dw \\ + \tfrac{1}{2} \Dw^\top \mathbf{H}^{(t-1)} \Dw    
    \end{alignedat}
\end{align}


As $\J(\whT{1}{t-1} + \Dw)$ is a good proxy of this sensitivity, we use it as a surrogate objective function. However, this quadratic approximation captures how the loss changes as we move away from $\whT{1}{t-1}$. However, there are no terms in this equation Eq.~\eqref{eq:taylor-exp} which drags the weights towards the current local model $\wlT{t}$ which has mastered the new task group. To incorporate the local model's knowledge, we add a regularization term that penalizes deviation from the target $\wlT{t}$. Defining $\Dd^{(t)} = \wlT{t} - \whT{1}{t-1}$ as the gap between the local and hierarchical models, the full surrogate objective becomes: 

\begin{equation}
    \label{eq:surrogate}
    \begin{alignedat}{2}
        &\min_{\Dw} \quad \underbrace{\J(\whT{1}{t-1}) + {\mathbf{g}^{(t-1)}}^\top \Dw + \tfrac{1}{2} \Dw^\top \mathbf{H}^{(t-1)} \Dw}_{\text{Second-order approximation of loss on past tasks}} \\ 
           & + \underbrace{\tfrac{\reg}{2} \left\| (\whT{1}{t-1} + \Dw)- \wlT{t} \right\|^2}_{\text{Pull toward local model}}
        \\
        &\Rightarrow \min_{\Dw} \quad \J(\whT{1}{t-1}) + {\mathbf{g}^{(t-1)}}^\top \Dw \\  &+ \tfrac{1}{2} \Dw^\top \mathbf{H}^{(t-1)} \Dw 
        + \tfrac{\reg}{2} \left\| \Dw - \Dd^{(t)} \right\|^2,
    \end{alignedat}
\end{equation}

where the hyperparameter $\reg > 0$ controls the strength of the pull toward the local model. When $\reg$ is large, the update prioritizes matching the local model; when $\reg$ is small, the update prioritizes staying in low-curvature regions. The closed-form minimizer of Eq.~\eqref{eq:surrogate}, $\Dw^{*}$, can be derived as: 

\begin{equation}
\label{eq:deltaW-main}
\Dw^{*} = \left(\mathbf{H}^{(t-1)} + \reg \mathbf{I}\right)^{-1}\!\left(\reg \, \Dd^{(t)} - \mathbf{g}^{(t-1)}\right).
\end{equation}

By applying this optimal step to the current weights, we arrive at the updated hierarchical weight update given by 

\begin{equation}
\label{eq:hier_update-main}
\whT{1}{t} = \whT{1}{t-1} + \left(\mathbf{H}^{(t-1)} + \reg \mathbf{I}\right)^{-1}\!\left(\reg \, \Dd^{(t)} - \mathbf{g}^{(t-1)}\right).
\end{equation}

We formalize this result in Theorem~\ref{prop:main}, which establishes that Eq.~\eqref{eq:deltaW-main} is the unique global minimizer of the surrogate objective defined in Eq.~\eqref{eq:surrogate}, provided that the regularized Hessian is positive definite.

\textbf{Interpretation of the update rule.} The update in \eqref{eq:hier_update-main} can be understood through its two components. The term $\reg \, \Dd^{(t)}$ pulls the hierarchical model toward the local model, encoding new task knowledge, while the term $-\mathbf{g}^{(t-1)}$ pushes the hierarchical model in the direction that reduces loss on past tasks. The matrix $(\mathbf{H}^{(t-1)} + \reg \mathbf{I})^{-1}$ then rescales this combined update. 

\begin{algorithm}[!ht]
\caption{HTCL}
\label{alg:task-ordering}
\begin{algorithmic}[1]
\REQUIRE Task set $\T = \{\tsk_1, \ldots, \tsk_\ntasks\}$, group size $\gsize$, regularization $\reg$, catch-up iterations $\ncatch$
\STATE Initialize $\whT{1}{0} \gets \text{random}$
\FOR{each $t = 1, 2, \ldots, \ngroups$ with task group $\G_t \subset \T$, $|\G_t| = \gsize$}
    \STATE Initialize local models: $\wlT{\gperm} \gets \whT{1}{t-1}$ for all $\gperm \in \gpermset_t$
    \FOR{each permutation $\gperm \in \gpermset_t$}
        \STATE Train $\wlT{\gperm} \gets \TrainSeq(\gperm, \whT{1}{t-1})$
        \STATE Evaluate $\Acc^{\gperm} \gets \Eval(\wlT{\gperm})$
    \ENDFOR
    \STATE Select best: $\gperm^{*} = \argmax_{\gperm} \Acc^{\gperm}$
    \STATE $\wlT{t} \gets \wlT{\gperm^{*}}$
    \IF{$t = 1$}
        \STATE $\whT{1}{1} \gets \wlT{1}$
    \ELSE
        \STATE $\mathbf{g}^{(t-1)} \gets \nabla \loss(\whT{1}{t-1})$
        \STATE $\mathbf{H}^{(t-1)} \gets \nabla^2 \loss(\whT{1}{t-1})$
        \STATE $\whT{1}{t} \gets \whT{1}{t-1} + (\mathbf{H}^{(t-1)} + \reg \mathbf{I})^{-1}\!\left[\reg(\wlT{t} - \whT{1}{t-1}) - \mathbf{g}^{(t-1)}\right]$
    \ENDIF
\ENDFOR
\FOR{$i = 1, \ldots, \ncatch$} 
    \STATE $\mathbf{g} \gets \nabla \loss(\whT{1}{\ngroups})$
    \STATE $\mathbf{H} \gets \nabla^2 \loss(\whT{1}{\ngroups})$
    \STATE $\whT{1}{\ngroups} \gets \whT{1}{\ngroups} + (\mathbf{H} + \reg \mathbf{I})^{-1}\!\left[\reg(\wlT{\ngroups} - \whT{1}{\ngroups}) - \mathbf{g}\right]$
\ENDFOR
\STATE \textbf{return} $\whT{1}{\ngroups}$
\end{algorithmic}
\end{algorithm}
\vspace{-4mm}

\textbf{Catch-up phase.} Since the update in \eqref{eq:hier_update-main} relies on a Taylor-series approximation rather than direct optimization, the hierarchical model may underfit the most recently learned tasks. To mitigate this, we introduce a \emph{catch-up phase} in which the hierarchical model is refined for $\ncatch$ additional iterations using the same Taylor-based update. In practice, a small number of iterations suffices with minimal computational overhead. Algorithm~\ref{alg:task-ordering} summarizes the full training procedure.  


The update rule in Eq.~\eqref{eq:hier_update-main} uses a second-order expansion around the current hierarchical state $\whT{1}{t-1}$. A natural question is whether we should account for interactions with groups further in the past. In principle, one could construct a multi-point Taylor expansion that jointly considers $\whT{1}{t-1}, \whT{1}{t-2}, \ldots$. However, such higher-order terms introduce substantial computational overhead. Figure~\ref{fig:htcl-overview} illustrates the complete HTCL framework, showing how tasks are partitioned into groups, how local models explore intra-group orderings, and how the hierarchical model integrates these local findings through second-order hierarchical updates. We now extend our two model formulation to $L$-layer setup in the following section.

\begin{figure}[!ht]
    \centering
    \includegraphics[width=\linewidth]{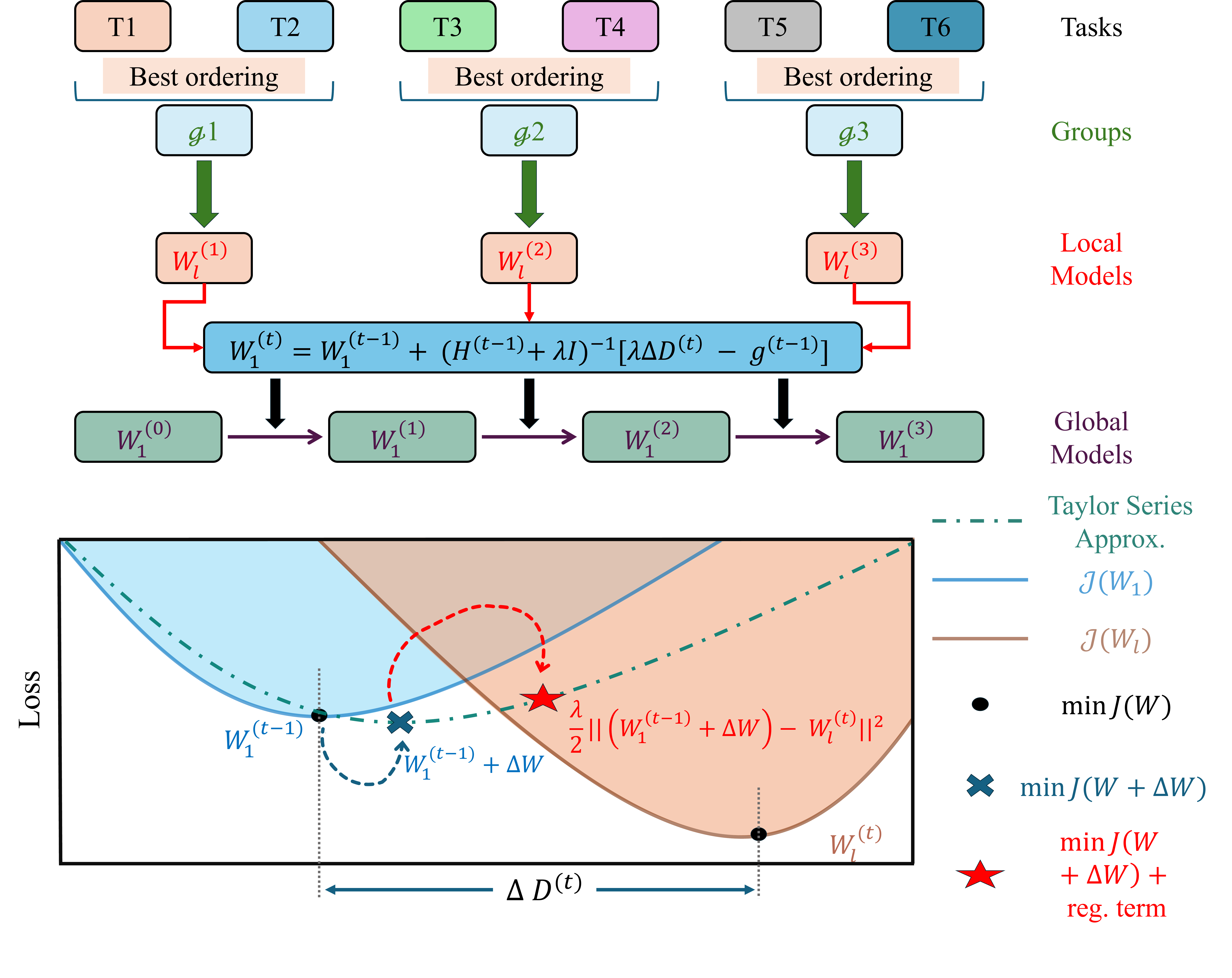}
    \caption{(top)~The HTCL framework for a sequence of six tasks partitioned into three groups ($\G_1, \G_2, \G_3$). Within each group, all $\gsize!$ (here $2!$) intra-group permutations are evaluated through neural network training to provide $\wlT{1}, \wlT{2}, \wlT{3}.$ These three results are successively integrated into the hierarchical model $\whT{1}{t}$. (bottom)~Illustration of the update rule on a loss landscape. The current hierarchical model $\whT{1}{t-1}$ (blue) integrates knowledge from the local model $\wlT{t}$ (brown). Rather than naively minimizing $\J(\w)$, HTCL approximates $\J(\whT{1}{t-1}+\Dw)$ (blue cross) using a second-order Taylor expansion of the loss around $\whT{1}{t-1}$ with a regularization term that penalizes deviation from $\wlT{t}$. 
    This yields the updated hierarchical model $\whT{1}{t}$ (red star).}
    \label{fig:htcl-overview}
\end{figure}

\begin{figure*}[!htbp]
    \centering
    \includegraphics[width=\textwidth]{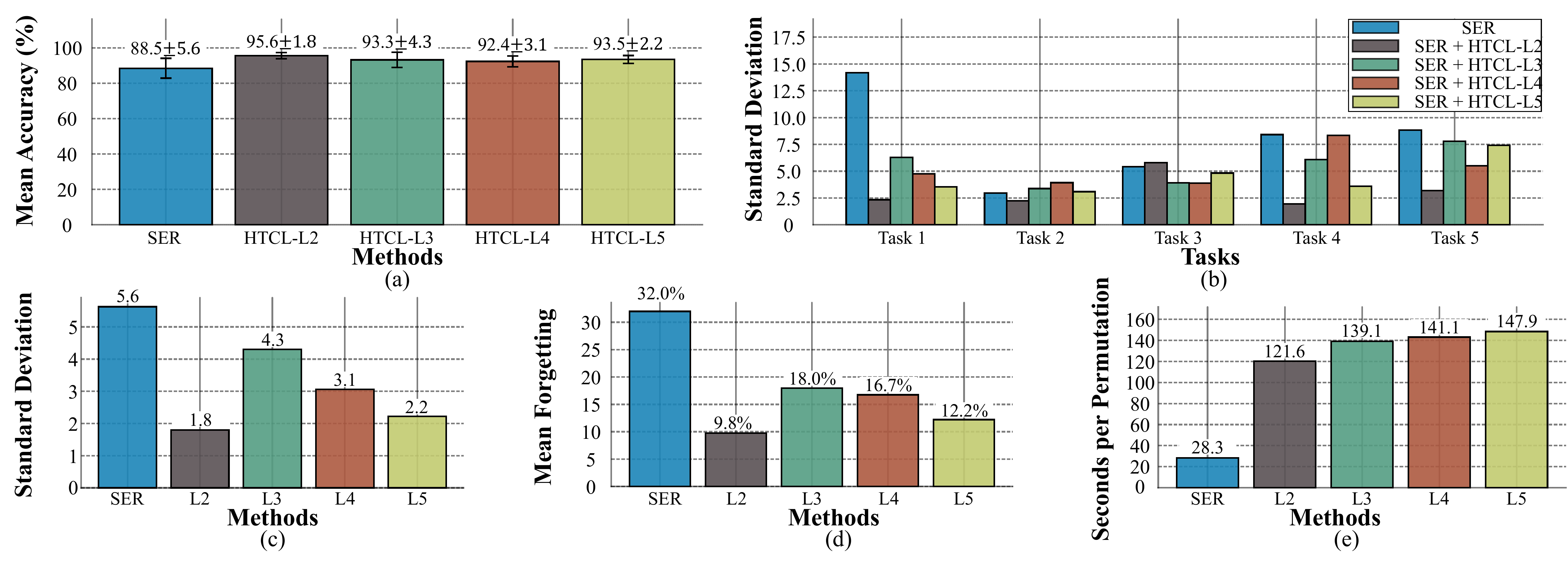}
    \caption{\textbf{Task-order robustness and memory retention of HTCL with SER on SplitMNIST.} Accuracy distributions are computed over all $120$ permutations of the five-task sequence. (a) Mean accuracy across all tasks after sequential training, comparing SER alone against SER augmented with 2-, 3-, 4-, and 5-level HTCL hierarchies. Error bars indicate standard deviation. (b) Per-task standard deviation across permutations. HTCL consistently achieves lower or comparable variance for most tasks across all hierarchy depths. (c) Overall standard deviation across methods, demonstrating improved task-order robustness with HTCL. (d) Mean forgetting across methods. HTCL reduces forgetting and improves memory retention. (e) Computation time per permutation. HTCL incurs moderate overhead relative to the baseline, reflecting a practical trade-off between stability and computational cost.}
    \label{fig:ser-results}
\end{figure*}

\subsection{Multi-Level Hierarchical Extension}

The preceding formulation uses a two-level architecture: one local model $\wl$ and one hierarchical model $\wh{1}$. This can be generalized to deeper hierarchies for finer-grained consolidation of knowledge across multiple timescales.  In an $\nlev$-level hierarchy, we maintain models $\wh{1}, \wh{2}, \ldots, \wh{\nlev}$ in addition to the local model $\wl$. The indexing reflects proximity to the local model: $\wh{1}$ is closest to $\wl$ and most responsive to recent changes, while $\wh{\nlev}$ is farthest and most stable. Each level integrates updates from the level immediately below it, with $\wh{1}$ integrating knowledge from the local model $\wl$, and $\wh{i}$ integrating knowledge from $\wh{i-1}$ for $i = 2, \ldots, \nlev$. The final output of HTCL is $\wh{\nlev}$, which serves as the ultimate learner.

For an $\nlev$-level hierarchy, if each level $i$ maintains its own curvature matrix $\mathbf{H}_i^{(t)}$ and gradient $\mathbf{g}_i^{(t)}$, then the update rule for level $i$ is:
\begin{equation}
\label{eq:multi_hierarchy}
\whT{i}{t} = \whT{i}{t-1} + \left(\mathbf{H}_i^{(t-1)} + \reg_i \mathbf{I}\right)^{-1}\!\left[\reg_i \, \Dd_i^{(t)} - \mathbf{g}_i^{(t-1)}\right]
\end{equation}
for $i = 1, \ldots, \nlev$, where:
\begin{itemize}
    \item For $i = 1$: $\Dd_1^{(t)} = \wlT{t} - \whT{1}{t-1}$ (difference from local model)
    \item For $i > 1$: $\Dd_i^{(t)} = \whT{i-1}{t} - \whT{i}{t-1}$ (difference from level below)
\end{itemize}

The final model $\w_L$ integrates information from all $L$ timescales with each level acting as a buffer that progressively mitigates the influence of any single group's task ordering on the final solution.

Each hierarchical level progressively smooths and consolidates information from the faster layers below it. Lower levels (smaller $i$) retain adaptability to recent tasks, while higher levels (larger $i$) maintain long-term stability through slower updates. Empirically, we find that deeper hierarchies (e.g., $\nlev=3$) provide improved resistance to catastrophic forgetting, particularly in long task sequences. Quantitative results and ablation studies are provided in Section~\ref{sec: results}.

\begin{remark}[Time and space complexity]
\label{rem:complexity}
The exact curvature-aware update in \eqref{eq:hier_update-main} requires forming and inverting a $\pdim \times \pdim$ Hessian, resulting in worst-case time $\Theta(\pdim^3)$ and memory $\Theta(\pdim^2)$ (or $\Theta(\nlev \pdim^3)$ and $\Theta(\nlev \pdim^2)$ for an $\nlev$-level hierarchy). In practice, HTCL employs scalable second-order approximations (e.g., diagonal Hessian approximation and group size smaller than $5$) which reduce both time and memory to near-linear. Complete derivations and comparisons with other CL methods are provided in Appendix~\ref{app:complexity}.
\end{remark}

%% file: sections/experiments.tex
\section{Experiments}

We evaluated HTCL on (i) two image classification benchmarks: SplitMNIST~\cite{lecun2002gradient} (10 classes divided into 5 tasks, 2 classes per task) and Split CIFAR-100~\cite{krizhevsky2009learning} (100 classes divided into 10 tasks, 10 classes per task); (ii) a graph node classification problem using Split Cora dataset~\cite{mccallum2000automating}; and (iii) a text classification problem on 20 Newsgroups~\cite{lang1995newsweeder} dataset  to demonstrate HTCL's generality across domains. All the experiments were conducted under the \textit{domain-incremental} learning setting, where task boundaries are known during training but task identity is not provided at test time. All experiments are performed on an NVIDIA A100 GPU.

We compared the HTCL method against top seven baseline CL methods spanning replay-based, regularization-based, and hybrid approaches: Strong Experience Replay (SER)~\cite{zhuo2023continual}, Dark Experience Replay (DER)~\cite{buzzega2020dark}, Experience Replay (ER), DualNet~\cite{pham2021dualnet}, iCaRL~\cite{rebuffi2017icarl}, EWC~\cite{kirkpatrick2017overcoming}, and Spectral Regularizer (SR)~\cite{ICLR2025_5565ab68}. These baselines were selected to include methods with state-of-the-art performance on our benchmarks (SER), foundational approaches in CL (EWC, ER), and recently proposed techniques (SR). In our framework, each baseline trains the \textit{local} model, while the \textit{global} model is updated via the Taylor series–based rule in Eq.~\eqref{eq:hier_update-main}. We calculate \textit{average accuracy}, \textit{average forgetting} and \textit{standard deviation} of accuracy and forgetting to demonstrate our results~\cite{lopez2017gradient, chaudhry2018efficient}. A detailed list of hyperparameters and other details on experimental setup can be found in Appendix~\ref{app:extended-results}.

\subsection{Results and Discussion}\label{sec: results}
\begin{figure}[!htbp]
    \centering
    \includegraphics[width=\linewidth]{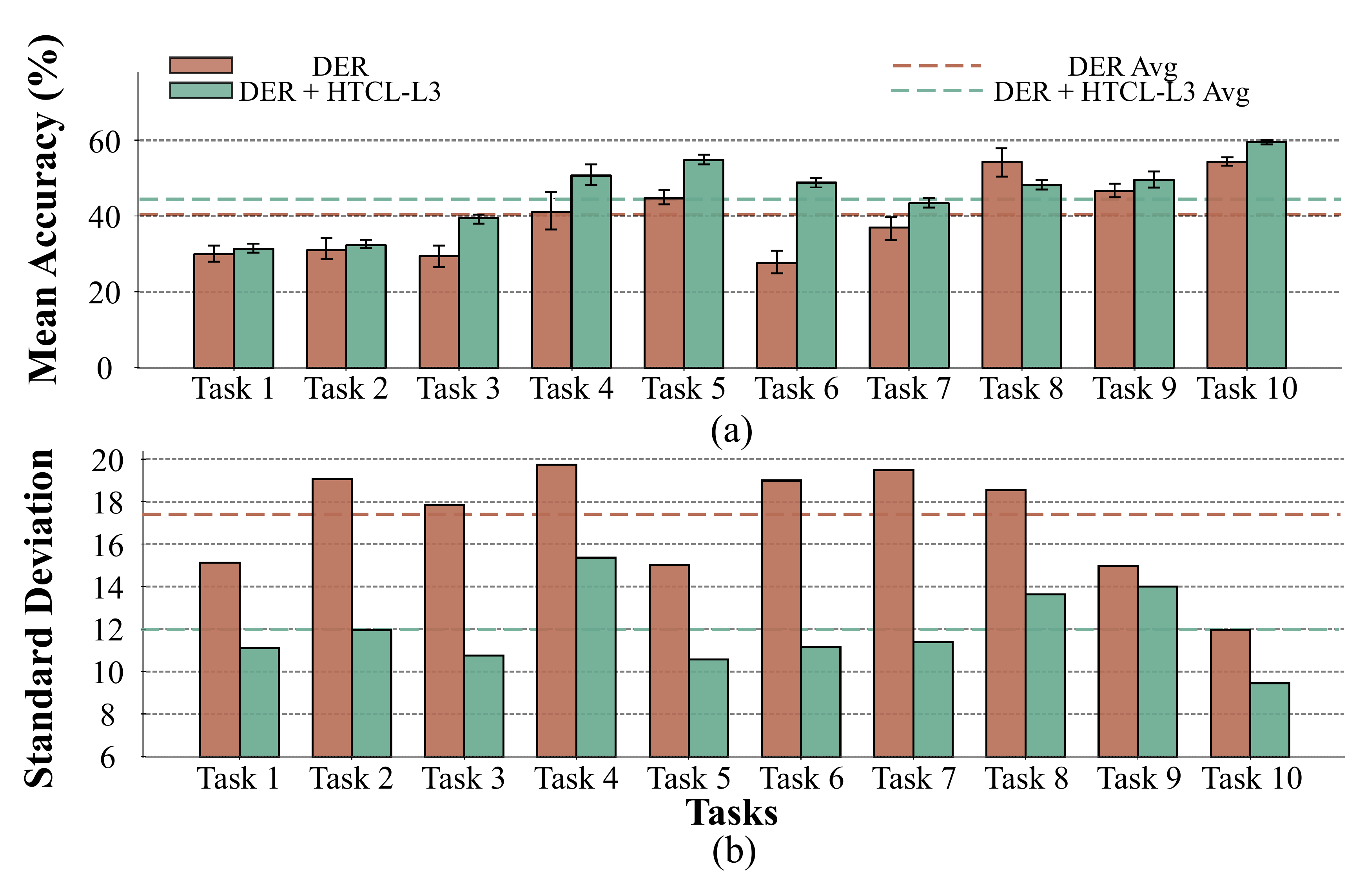}
    \caption{\textbf{Scalability of HTCL under long task sequences.} We extend the two-level local–global structure to a three-level hierarchy as described in Eq.~\eqref{eq:multi_hierarchy}, using DER as a baseline for illustration. The three-layer HTCL demonstrates higher memory retention and lower variance when handling long task sequences. (a) Mean accuracy per task: DER + HTCL-L3 achieves $45.2\%$ compared to $40.6\%$ for DER alone. (b) Standard deviation per task: DER + HTCL-L3 maintains a lower standard deviation ($12.0$) compared to DER alone ($17.4$).}
    \label{fig:extended}
\end{figure}

We first present representative results illustrating the performance of HTCL on SplitMNIST (see Figure~\ref{fig:ser-results}) using SER as a baseline through mean accuracy, task-order sensitivity, forgetting, and computational overhead. Next, we demonstrate the scalability of our HTCL to extended task sequences using DER on Split CIFAR-100 (see Figure~\ref{fig:extended}). A comprehensive comparison against seven baseline models across four diverse datasets (SplitMNIST, Split CIFAR-100, Split Cora, 20 Newsgroups) can be found in Table~\ref{tab:mean-std}. We further investigate the effects of memory buffer size on HTCL performance in Appendix~\ref{app:small-buffer}. Finally, since our consolidation method resembles some similarities with federated aggregation, we present a comparative analysis against FedProx~\cite{li2020federated} and FedAvg~\cite{mcmahan2017communication} in Appendix~\ref{app:federated}.

\textbf{Enhanced task-order robustness.} Panels~(b) and~(c)  of Figure~\ref{fig:ser-results} characterize stability with respect to task-order permutations. Panel~(b) shows that HTCL enables lower per-task variance, with reductions of up to $82\%$ on the most volatile tasks (e.g., Task 1), suggesting that individual tasks are less sensitive to the order in which they are encountered. Panel~(c) summarizes this effect at the sequence level, where HTCL markedly reduces the overall standard deviation by $67.9\%$ (dropping from $5.6$ to $1.8$) across permutations, with HTCL-L2 exhibiting the strongest robustness. This stabilizing effect generalizes effectively to other modalities; on Cora (graph) and 20 Newsgroups (text), HTCL reduces standard deviation by approximately $38\%$ and $33\%$ respectively, confirming that the method's robustness is not limited to image domains (results are shown in Table~\ref{tab:mean-std}).

\textbf{Accuracy and memory retention.} Panel~(a) in Figure~\ref{fig:ser-results} shows that incorporating the Taylor-based consolidation consistently improves mean accuracy relative to the SER baseline, with HTCL-L2 achieving the largest gain of $7.1$ percentage points ($95.6\%$ vs $88.5\%$). We observe similar improvements in non-image tasks, where HTCL boosts mean accuracy by $5.8$ percentage points on Cora and $6.3$ percentage points on 20 Newsgroups. This improvement reflects stronger retention of earlier tasks rather than overfitting to the final task, a trend that is corroborated by the forgetting metric in Panel~(d), where HTCL-L2 exhibits a drastic drop in forgetting compared to the baseline ($9.8\%$ vs $32.0\%$). This $22.2\%$ reduction on SplitMNIST is mirrored in the graph and text domains, where forgetting drops by over $10$ percentage points relative to the baseline. These results indicate that the global consolidation step effectively preserves previously learned representations, preventing the catastrophic overwriting of weights associated with earlier tasks. Together, these results demonstrate that the observed accuracy gains stem principally from improved long-term memory retention.

\begin{table*}[!htbp]
\centering
\caption{Final accuracy and variability across domains under the domain-incremental setting using a replay buffer of 50 samples for SplitMNIST, Cora, and 20 Newsgroups, and 500 samples for CIFAR-100. Mean accuracy, standard deviation, and mean forgetting are computed over 120 task-order permutations for SplitMNIST and 20 Newsgroups, 6 permutations for Cora, and 200 random permutations for CIFAR-100. Across all methods, incorporating HTCL consistently reduces task-order sensitivity, yielding substantially lower standard deviations while preserving or improving mean accuracy. $\uparrow$ denotes an increase and $\downarrow$ denotes a decrease relative to the baseline. Best results within each method group are shown in \textbf{bold}.}
\label{tab:mean-std}
\resizebox{\textwidth}{!}{%
\begin{tabular}{lcccccccccccc}
\hline
\multirow{2}{*}{\textbf{Method}} &
  \multicolumn{12}{c}{\textbf{Datasets}} \\ \cline{2-13} 
 &
  \multicolumn{3}{c}{\textbf{SplitMNIST}} &
  \multicolumn{3}{c}{\textbf{CIFAR-100}} & 
  \multicolumn{3}{c}{\textbf{Cora}} & 
  \multicolumn{3}{c}{\textbf{20 Newsgroups}} \\ \hline
 &
  \textbf{Mean Acc.} &
  \textbf{Std} &
  \textbf{Mean Forget} &
  \textbf{Mean Acc.} &
  \textbf{Std} &
  \textbf{Mean Forget} &
  \textbf{Mean Acc.} &
  \textbf{Std} &
  \textbf{Mean Forget} &
  \textbf{Mean Acc.} &
  \textbf{Std} &
  \textbf{Mean Forget} \\ \hline
SER &
  88.5 & 5.6 & 32.0 &
  42.62 & 18.73 & 56.5 &
  72.5 & 6.8 & 28.5 &
  58.2 & 8.4 & 35.2 \\
SER + HTCL-L2 &
  \textbf{95.6}$\uparrow$ & \textbf{1.8}$\downarrow$ & \textbf{9.8}$\downarrow$ &
  44.7$\uparrow$ & 14.9$\downarrow$ & 40.0$\downarrow$ &
  \textbf{78.3}$\uparrow$ & \textbf{4.2}$\downarrow$ & \textbf{18.2}$\downarrow$ &
  \textbf{64.5}$\uparrow$ & \textbf{5.6}$\downarrow$ & \textbf{24.8}$\downarrow$ \\
SER + HTCL-L3 &
  93.3$\uparrow$ & 4.3$\downarrow$ & 18.0$\downarrow$ &
  \textbf{46.2}$\uparrow$ & \textbf{12.5}$\downarrow$ & \textbf{30.0}$\downarrow$ &
  76.8$\uparrow$ & 5.1$\downarrow$ & 21.4$\downarrow$ &
  62.8$\uparrow$ & 6.3$\downarrow$ & 27.5$\downarrow$ \\ \hline
DER &
  86.5 & 4.9 & 35.0 &
  40.6 & 17.4 & 58.2 &
  70.2 & 7.2 & 30.8 &
  55.8 & 9.1 & 38.4 \\
DER + HTCL-L2 &
  \textbf{95.1}$\uparrow$ & \textbf{1.8}$\downarrow$ & \textbf{10.2}$\downarrow$ &
  43.0$\uparrow$ & 14.4$\downarrow$ & 42.6$\downarrow$ &
  \textbf{77.5}$\uparrow$ & \textbf{4.5}$\downarrow$ & \textbf{19.5}$\downarrow$ &
  \textbf{62.3}$\uparrow$ & \textbf{6.0}$\downarrow$ & \textbf{26.5}$\downarrow$ \\
DER + HTCL-L3 &
  92.3$\uparrow$ & 4.3$\downarrow$ & 18.8$\downarrow$ &
  \textbf{45.2}$\uparrow$ & \textbf{12.0}$\downarrow$ & \textbf{33.0}$\downarrow$ &
  75.2$\uparrow$ & 5.4$\downarrow$ & 22.8$\downarrow$ &
  60.5$\uparrow$ & 6.8$\downarrow$ & 29.2$\downarrow$ \\ \hline
ER &
  89.3 & 5.5 & 30.3 &
  38.0 & 20.0 & 60.0 &
  68.5 & 7.8 & 32.5 &
  54.4 & 9.8 & 42.0 \\
ER + HTCL-L2 &
  \textbf{92.0}$\uparrow$ & \textbf{3.5}$\downarrow$ & \textbf{12.5}$\downarrow$ &
  41.0$\uparrow$ & 17.0$\downarrow$ & 48.0$\downarrow$ &
  \textbf{74.2}$\uparrow$ & \textbf{5.2}$\downarrow$ & \textbf{22.0}$\downarrow$ &
  \textbf{58.5}$\uparrow$ & \textbf{6.8}$\downarrow$ & \textbf{32.5}$\downarrow$ \\
ER + HTCL-L3 &
  91.5$\uparrow$ & 4.0$\downarrow$ & 14.0$\downarrow$ &
  \textbf{43.5}$\uparrow$ & \textbf{15.5}$\downarrow$ & \textbf{40.0}$\downarrow$ &
  73.0$\uparrow$ & 5.8$\downarrow$ & 24.5$\downarrow$ &
  57.2$\uparrow$ & 7.4$\downarrow$ & 34.0$\downarrow$ \\ \hline
iCaRL &
  93.2 & 4.8 & 12.5 &
  41.8 & 17.2 & 54.4 &
  74.8 & 6.2 & 22.5 &
  60.5 & 7.5 & 30.2 \\
iCaRL + HTCL-L2 &
  \textbf{96.1}$\uparrow$ & \textbf{2.1}$\downarrow$ & \textbf{6.8}$\downarrow$ &
  44.2$\uparrow$ & 14.0$\downarrow$ & 42.5$\downarrow$ &
  \textbf{78.2}$\uparrow$ & \textbf{3.8}$\downarrow$ & \textbf{15.8}$\downarrow$ &
  \textbf{65.8}$\uparrow$ & \textbf{5.0}$\downarrow$ & \textbf{22.5}$\downarrow$ \\
iCaRL + HTCL-L3 &
  95.0$\uparrow$ & 3.4$\downarrow$ & 9.2$\downarrow$ &
  \textbf{46.6}$\uparrow$ & \textbf{11.8}$\downarrow$ & \textbf{35.0}$\downarrow$ &
  77.5$\uparrow$ & 4.5$\downarrow$ & 18.2$\downarrow$ &
  64.0$\uparrow$ & 5.8$\downarrow$ & 25.0$\downarrow$ \\ \hline
DualNet &
  89.2 & 9.0 & 15.0 &
  40.6 & 17.5 & 57.0 &
  71.8 & 8.5 & 26.2 &
  56.5 & 10.2 & 33.5 \\
DualNet + HTCL-L2 &
  90.5$\uparrow$ & 6.0$\downarrow$ & 10.5$\downarrow$ &
  41.5$\uparrow$ & 14.4$\downarrow$ & 45.0$\downarrow$ &
  75.5$\uparrow$ & 5.8$\downarrow$ & 19.0$\downarrow$ &
  61.2$\uparrow$ & 7.2$\downarrow$ & 26.8$\downarrow$ \\
DualNet + HTCL-L3 &
  \textbf{91.0}$\uparrow$ & \textbf{5.5}$\downarrow$ & \textbf{9.0}$\downarrow$ &
  \textbf{44.0}$\uparrow$ & \textbf{12.0}$\downarrow$ & \textbf{36.0}$\downarrow$ &
  \textbf{76.8}$\uparrow$ & \textbf{5.2}$\downarrow$ & \textbf{17.5}$\downarrow$ &
  \textbf{62.5}$\uparrow$ & \textbf{6.5}$\downarrow$ & \textbf{24.5}$\downarrow$ \\ \hline
SR &
  88.0 & 7.5 & 31.0 &
  13.0 & 24.0 & 65.0 &
  69.5 & 8.8 & 35.0 &
  54.2 & 10.5 & 40.5 \\
SR + HTCL-L2 &
  90.0$\uparrow$ & 5.0$\downarrow$ & 14.0$\downarrow$ &
  18.5$\uparrow$ & 20.0$\downarrow$ & 55.0$\downarrow$ &
  74.0$\uparrow$ & 6.0$\downarrow$ & 25.5$\downarrow$ &
  59.6$\uparrow$ & 7.5$\downarrow$ & 32.7$\downarrow$ \\
SR + HTCL-L3 &
  \textbf{91.0}$\uparrow$ & \textbf{4.0}$\downarrow$ & \textbf{12.0}$\downarrow$ &
  \textbf{22.0}$\uparrow$ & \textbf{18.0}$\downarrow$ & \textbf{48.0}$\downarrow$ &
  \textbf{75.5}$\uparrow$ & \textbf{5.2}$\downarrow$ & \textbf{22.0}$\downarrow$ &
  \textbf{61.0}$\uparrow$ & \textbf{6.8}$\downarrow$ & \textbf{29.1}$\downarrow$ \\ \hline
EWC &
  79.3 & 9.8 & 42.8 &
  12.5 & 23.5 & 68.0 &
  65.2 & 9.7 & 38.2 &
  48.5 & 11.1 & 45.8 \\
EWC + HTCL-L2 &
  85.2$\uparrow$ & 6.2$\downarrow$ & 28.9$\downarrow$ &
  16.8$\uparrow$ & 19.5$\downarrow$ & 58.0$\downarrow$ &
  70.9$\uparrow$ & 6.5$\downarrow$ & 28.6$\downarrow$ &
  54.7$\uparrow$ & 8.8$\downarrow$ & 35.5$\downarrow$ \\
EWC + HTCL-L3 &
  \textbf{87.0}$\uparrow$ & \textbf{5.0}$\downarrow$ & \textbf{24.0}$\downarrow$ &
  \textbf{19.3}$\uparrow$ & \textbf{17.1}$\downarrow$ & \textbf{52.5}$\downarrow$ &
  \textbf{72.5}$\uparrow$ & \textbf{5.8}$\downarrow$ & \textbf{25.5}$\downarrow$ &
  \textbf{56.0}$\uparrow$ & \textbf{7.2}$\downarrow$ & \textbf{33.3}$\downarrow$ \\ \hline
\end{tabular}%
}
\end{table*}

\textbf{Practical computational overhead.} Panel~(e) reports the computation time per permutation. HTCL introduces additional cost, increasing the time per permutation from $28.3$ seconds (SER) to $121.6$ seconds (HTCL-L2). However, this overhead scales gradually with hierarchy depth. In practice, this trade-off is highly favorable: while computational cost increases by approximately $4\times$, it yields a $69\%$ relative reduction in forgetting ($32.0\%$ to $9.8\%$) and a $68\%$ reduction in task-order variance ($5.6$ to $1.8$). This makes the method practical for CL scenarios where reliability and robustness to random task order are priorities over raw training speed.

\textbf{Generalization across methods.} Table~\ref{tab:mean-std} reports mean accuracy, standard deviation, and mean forgetting across multiple CL baselines under the domain-incremental setting. Across all baseline methods, incorporating HTCL consistently reduces standard deviation, demonstrating improved task-order robustness. For the strongest baselines such as SER and DER, HTCL also yields substantial gains in mean accuracy and reductions in forgetting, confirming that the Taylor-based consolidation rule is broadly complementary to existing approaches. The consistent reduction in standard deviation across all methods highlights HTCL's role as a variance-reducing mechanism that stabilizes learning regardless of the underlying baseline.  


\textbf{Scaling to long task sequences.} To assess how HTCL scales to longer task horizons, we extend the two-level local–global framework to a three-level hierarchy following the Eq.~\eqref{eq:multi_hierarchy}. The additional intermediate level introduces a consolidation stage that operates between the fast-adapting local model and the slowly evolving global model, progressively smoothing updates across timescales.

We evaluated this configuration on Split CIFAR-100, a challenging benchmark consisting of 100 classes divided into 10 groups of sequential tasks. Figure~\ref{fig:extended} presents results using DER as the base method. The benefits of hierarchical consolidation are pronounced. Panel (a) shows that DER + HTCL-L3 achieves $45.2\%$ mean accuracy across all tasks, representing a $5$ percentage point improvement over standalone DER ($40.6\%$).

Beyond accuracy improvements, the three-level hierarchy substantially reduces task-order sensitivity. Panel (b) shows that per-task standard deviation drops from $17.4$ for DER to $12.0$ for DER + HTCL-L3, a reduction of over $30\%$. This indicates that hierarchical consolidation not only improves retention but also yields more predictable performance regardless of task arrival order. Furthermore, panel (a) demonstrates that the HTCL with 3 hierarchical levels effectively retains performance better than the baseline and HTCL with 2 hierarchical levels (see CIFAR-100 dataset in Table~\ref{tab:mean-std}), preventing the sharp accuracy drops observed in the baseline for mid-sequence tasks (e.g., Tasks 3-7).

These results confirm that deeper hierarchies provide tangible benefits for extended task sequences, where the accumulated interference across many tasks overwhelms single-model learners. By distributing consolidation across multiple temporal scales, HTCL maintains stability without sacrificing the plasticity needed to learn new tasks.

%% file: sections/conclusion.tex
\section{Conclusion}\label{sec:conclusion}




In this work, we propose an alternative approach to mitigating task-order sensitivity in CL without depending on exhaustive combinatorial search over task permutations. Our framework decouples the learning process into local exploration for identifying the best intra-group performances and a hierarchical consolidation mechanism that integrates these localized solutions to achieve global robustness to task order. By combining hierarchical structure with second-order updates, HTCL enables a principled strategy of knowledge consolidation across multiple timescales. We believe that this perspective provides a practical path toward scalable continual learners that remain reliable under long task horizons while being compatible with existing CL methods. 

%% file: sections/appendix_v2.tex
\appendix
\onecolumn
\section{Supplementary Materials}

\input{sections/related_work}

\subsection{Consolidation Equation~\eqref{eq:hier_update-main}}
\label{app:update-proof}

\begin{theorem}
\label{prop:main}
If $\mathbf{H}^{(t-1)} + \reg \mathbf{I} \succ 0$, the minimizer of the surrogate objective \eqref{eq:surrogate} is:
\begin{equation}
\label{eq:deltaW}
\Dw^{*} = \left(\mathbf{H}^{(t-1)} + \reg \mathbf{I}\right)^{-1}\!\left(\reg \, \Dd^{(t)} - \mathbf{g}^{(t-1)}\right)
\end{equation}
and the updated hierarchical weights are:
\begin{equation}
\label{eq:hier_update}
\whT{1}{t} = \whT{1}{t-1} + \left(\mathbf{H}^{(t-1)} + \reg \mathbf{I}\right)^{-1}\!\left(\reg \, \Dd^{(t)} - \mathbf{g}^{(t-1)}\right)
\end{equation}
where $\Dd^{(t)} = \wlT{t} - \whT{1}{t-1}$. The surrogate is strictly convex and $\Dw^{*}$ is the global minimizer.
\end{theorem}

\begin{proof}

Taking the gradient of \eqref{eq:surrogate} with respect to $\Dw$ and setting it to zero:

\begin{align*}
    \nabla_{\Dw} &\Bigg[ \underbrace{\J(\whT{1}{t-1})}_{\text{const.}} + {\mathbf{g}^{(t-1)}}^\top \Dw + \tfrac{1}{2} \Dw^\top \mathbf{H}^{(t-1)} \Dw 
    + \tfrac{\reg}{2} \left\| \Dw - \Dd^{(t)} \right\|^2 \Bigg] = 0 \\
    &\Rightarrow \mathbf{g}^{(t-1)} + \mathbf{H}^{(t-1)} \Dw + \reg \left(\Dw - \Dd^{(t)}\right) = 0 \\
    &\Rightarrow \left(\mathbf{H}^{(t-1)} + \reg \mathbf{I}\right) \Dw = \reg \, \Dd^{(t)} - \mathbf{g}^{(t-1)} 
\end{align*}

If the smallest eigenvalue of $\mathbf{H}^{(t-1)}$ is denoted by $\mu_{\min}(\mathbf{H}^{(t-1)})$, then the eigenvalues of $\mathbf{H}^{(t-1)} + \reg \mathbf{I}$ are shifted by $\reg$, giving
\begin{equation*}
    \mu_{\min}\left(\mathbf{H}^{(t-1)} + \reg \mathbf{I}\right) = \mu_{\min}\left(\mathbf{H}^{(t-1)}\right) + \reg.
\end{equation*}
Therefore, to guarantee positive definiteness, $\reg$ must satisfy
\begin{equation*}
    \reg > -\mu_{\min}\left(\mathbf{H}^{(t-1)}\right).
\end{equation*}
If this condition holds, then $\mathbf{H}^{(t-1)} + \reg \mathbf{I} \succ 0$ and is invertible, ensuring the surrogate is strictly convex with a global minimizer and we obtain the state solution. 
\end{proof}

\subsection{Recursive Formulation of the Global Update Rule}
\label{app:recursive}

The recursive formulation presented in Lemma~\ref{lem:recursive} generalizes the one-step Taylor series-based update in Eq.~\eqref{eq:hier_update} to a two-point approximation. While the main paper focuses on the standard one-point second-order expansion around $\whT{1}{t-1}$, it may be  useful to understand how the cumulative effect of successive Taylor updates can be expressed in closed form. In particular, combining the updates at timestep $t-1$ and $t$ yields a two-point approximation in which the global model integrates information from two consecutive local models. This recursive view does not change the algorithm but highlights how HTCL accumulates second-order information over time and makes the additive structure inherent to sequential Taylor series expansions. 

\begin{lemma}[Recursive formulation]
\label{lem:recursive}
Suppose the global updates at consecutive timesteps follow the Eq.~\eqref{eq:hier_update} and assume each matrix $\mathbf{H}^{(s-1)}+\reg \mathbf{I}$ is invertible. Then the cumulative second-order integration of local models into the global weights across the two timesteps yields
\begin{equation}
\label{eq:recursive}
\begin{aligned}
\whT{1}{t} 
&= \whT{1}{t-2} 
+ (\mathbf{H}^{(t-2)} + \reg \mathbf{I})^{-1}\!\big[\reg \Dd^{(t-1)} - \mathbf{g}^{(t-2)}\big] \\
&\qquad + (\mathbf{H}^{(t-1)} + \reg \mathbf{I})^{-1}\!\big[\reg \Dd^{(t)} - \mathbf{g}^{(t-1)}\big]
\end{aligned}
\end{equation}
\end{lemma}

\begin{proof}
Update at time $t-1$ using the update rule:
\[
\whT{1}{t-1} = \whT{1}{t-2} + (\mathbf{H}^{(t-2)}+\reg \mathbf{I})^{-1}\!\big[\reg \Dd^{(t-1)} - \mathbf{g}^{(t-2)}\big]
\]
Now write the update at time $t$:
\[
\whT{1}{t} = \whT{1}{t-1} + (\mathbf{H}^{(t-1)}+\reg \mathbf{I})^{-1}\!\big[\reg \Dd^{(t)} - \mathbf{g}^{(t-1)}\big]
\]
Substitute the expression for $\whT{1}{t-1}$ into the update at time $t$:
\[
\boxed{
\begin{aligned}
\whT{1}{t}
&= \whT{1}{t-2} + (\mathbf{H}^{(t-2)}+\reg \mathbf{I})^{-1}\!\big[\reg \Dd^{(t-1)} - \mathbf{g}^{(t-2)}\big] \\
&\qquad + (\mathbf{H}^{(t-1)}+\reg \mathbf{I})^{-1}\!\big[\reg \Dd^{(t)} - \mathbf{g}^{(t-1)}\big]
\end{aligned}
}
\]
Re-arranging terms yields the claimed identity \eqref{eq:recursive}. The derivation requires only linear algebraic substitution and the invertibility of the matrices $(\mathbf{H}^{(s-1)}+\reg \mathbf{I})$ to justify the closed-form updates. 
\end{proof}

\subsection{Performance Bounds for the Algorithm~\ref{alg:task-ordering}}
\label{app:performance-bounds}

\begin{theorem}[Performance Bound]
\label{prop:group_vs_random}
Let $\T$ be a set of $\ntasks$ tasks partitioned into $\ngroups$ disjoint groups $\G_1, \ldots, \G_\ngroups$, each of size $\gsize$ such that  $\ntasks \approx \ngroups \cdot \gsize$) with $\gsize!$ permutations in each group. For each group $\G_j$, let $\gpermset_j$ denote the set of its $k!$ intra-group permutations, where each $\sigma \in \gpermset_j$ specifies an ordering of tasks within $\G_j$. Let the average accuracy of group $\G_j$ achieved with permutation $\sigma$ after training the Neural Network be
$\Acc_j^{(\gperm)} = \frac{1}{\gsize} \sum_{\tsk \in \G_j} \Eval_{\tsk}\!\left(\w^{(j)}\right),$
where $\Eval_{\tsk}(\w^{(j)})$ represents performance on $\tsk$ with $\gperm$. With the proposed HTCL, the optimal accuracy within a group is achieved by design (see Algorithm~\ref{alg:task-ordering}).

\noindent\textbf{Then the improvement over arbitrary/random selection:} For any selection of intra-group permutations $(\gperm_1, \ldots, \gperm_\ngroups)$ with $\gperm_j \in \gpermset_j$,$
    \sum_{j=1}^{\ngroups} \Acc^{\mathrm{HTCL}}(\G_j) \geq \sum_{j=1}^{\ngroups} \Acc_j^{(\gperm_j)}.$ In particular, if each $\gperm_j$ is drawn uniformly at random from $\gpermset_j$,
    \[
    \sum_{j=1}^{\ngroups} \Acc^{\mathrm{HTCL}}(\G_j) \geq \E_{\gperm_1, \ldots, \gperm_\ngroups}\!\left[\sum_{j=1}^{\ngroups} \Acc_j^{(\gperm_j)}\right].
    \]
\end{theorem}

\begin{proof}
\textbf{Lower Bound.} For each group $\G_j$ and any chosen permutation $\gperm_j \in \gpermset_j$, the definition of the maximum gives:
\[
\Acc^{\mathrm{HTCL}}(\G_j) = \max_{\gperm \in \gpermset_j} \Acc_j^{(\gperm)} \geq \Acc_j^{(\gperm_j)}
\]
Summing over all $\ngroups$ groups:
\[
\sum_{j=1}^{\ngroups} \Acc^{\mathrm{HTCL}}(\G_j) \geq \sum_{j=1}^{\ngroups} \Acc_j^{(\gperm_j)}
\]
This holds for any fixed choice of $(\gperm_1, \ldots, \gperm_\ngroups)$. Taking expectation over random selections $\gperm_j \sim \mathrm{Uniform}(\gpermset_j)$ preserves the inequality (since the left-hand side is deterministic):
\[
\boxed{\sum_{j=1}^{\ngroups} \Acc^{\mathrm{HTCL}}(\G_j) \geq \E_{\gperm_1, \ldots, \gperm_\ngroups}\!\left[\sum_{j=1}^{\ngroups} \Acc_j^{(\gperm_j)}\right]}
\]

\end{proof}

\subsection{Time and Space Complexity Analysis}
\label{app:complexity}

This appendix provides the detailed derivations supporting Remark~\ref{rem:complexity}.  
We analyze the exact and practical computational complexity of HTCL, including permutation search, 
curvature computation, and hierarchical integration, and compare these with  
standard single-model continual learning approaches.

\subsubsection{Exact Complexity of the Global Integration}

Recall that the global update for a single level is
\[
\whT{1}{t} 
= \whT{1}{t-1} 
+ (\mathbf{H}^{(t-1)} + \reg \mathbf{I})^{-1}\!\big[\reg(\wlT{t} - \whT{1}{t-1}) - \mathbf{g}^{(t-1)}\big].
\]

\begin{proposition}[Exact computational cost]
Forming and inverting the dense Hessian $\mathbf{H}^{(t-1)}\in\R^{\pdim\times \pdim}$ yields:
\[
\text{Time}_{\mathrm{exact}} = \Theta(\pdim^3), 
\qquad
\text{Memory}_{\mathrm{exact}} = \Theta(\pdim^2).
\]
\end{proposition}

\begin{proof}[Proof sketch]
A full Hessian contains $\pdim^2$ entries and computing all second derivatives requires  
$\Theta(\pdim^2)$ memory and up to $\Theta(\pdim^2 \!\cdot\! \text{samples})$ time.  
Solving $(\mathbf{H}+\reg \mathbf{I})x=b$ via dense Cholesky or Gaussian elimination requires  
$\Theta(\pdim^3)$ time.  
\end{proof}

\begin{corollary}[Hierarchy cost]
For an $\nlev$-level hierarchy:
\[
\text{Time}_{\mathrm{exact}}^{(\nlev)} = \Theta(\nlev \pdim^3),
\qquad 
\text{Memory}_{\mathrm{exact}}^{(\nlev)} = \Theta(\nlev \pdim^2).
\]
\end{corollary}

\subsubsection{Practical Complexity via Standard Approximations}

In practice, HTCL uses scalable second-order approximations, which reduce the cubic  
cost of exact Hessian inversion to near-linear time and memory.

\paragraph{Diagonal / Fisher approximation.}  
In our experiments we used this approximation. This approach stores only the diagonal of $\mathbf{H}$ or the Fisher information:
\[
\text{Time}_{\mathrm{diag}}=\Theta(\pdim),
\qquad 
\text{Memory}_{\mathrm{diag}}=\Theta(\pdim).
\]

\paragraph{Low-rank curvature approximation.}  
If $\mathbf{H}$ is approximated by a low-rank factorization of rank $r\ll \pdim$,  
then storing and applying the approximation requires:
\[
\text{Memory}_{\mathrm{low\text{-}rank}} = \Theta(r\pdim), 
\qquad
\text{Time}_{\mathrm{low\text{-}rank}} = \Theta(r\pdim).
\]

These are the two approximations used in our implementation.

\subsubsection{Cost of Permutation Search}

Let each task group contain $\gsize$ tasks.  
For each permutation $\perm\in \gpermset_\gsize$, the local model is trained for $N_{\text{iter}}$ minibatch updates, 
each costing $T_{\text{step}}$.  
Let $B$ denote the minibatch size.

\begin{proposition}[Exact permutation-search cost]
\[
\text{Time}_{\mathrm{perm,exact}}
= \gsize!\, N_{\text{iter}}\, T_{\text{step}}
= \Theta(\gsize!\, E \cdot \text{\#samples} \cdot \pdim).
\]
\end{proposition}

\begin{proof}[Proof sketch]
Each permutation requires $N_{\text{iter}}$ updates, each with forward--backward cost  
$T_{\text{step}} = \Theta(B\pdim)$.  
There are $\gsize!$ permutations.  
When models are evaluated sequentially, memory remains $\Theta(\pdim)$.  
\end{proof}

\paragraph{Practical setting.}  
We restrict the group size to $\gsize < 5$, so that  
\[
\gsize! \le 24
\]
is a small constant.  
Thus, the permutation-search overhead does not grow with the number of tasks and simplifies to:
\[
\text{Time}_{\mathrm{perm,practical}}
= \mathcal{O}(N_{\text{iter}}\,T_{\text{step}}),
\]
with no additional asymptotic dependence on $\gsize$.

\subsubsection{Total Practical Complexity of HTCL}

Combining bounded permutation search with curvature-regularized global integration gives:
\[
\boxed{
\text{Time}_{\mathrm{HTCL}} 
=
\underbrace{\mathcal{O}(N_{\text{iter}}\,T_{\text{step}})}_{\text{permutation search } (\gsize! \le 24)}
+
\underbrace{\mathcal{O}(\pdim)}_{\text{curvature update (diag or low-rank)}}
}
\]

\[
\boxed{
\text{Memory}_{\mathrm{HTCL}}
=
\Theta(\nlev \pdim)
+
\begin{cases}
\Theta(\pdim), & \text{(diagonal approximation)},\\[3pt]
\Theta(r\pdim), & \text{(low-rank approximation)}.
\end{cases}
}
\]

Typical configurations use  
$\gsize < 5$ (hence $\gsize!\le 24$),  
$r \ll \pdim$,  
and $\nlev \in \{2,3\}$,  
making HTCL comparable in cost to replay or EWC.

\subsubsection{Comparison with Standard Continual Learning Methods}

\paragraph{Regularization-based CL (EWC, SI, MAS).}  
Extra memory $\Theta(\pdim)$; extra per-step time $\Theta(\pdim)$. Lightweight but order-sensitive.

\paragraph{Replay-based CL (ER, DER, iCaRL).}  
Memory $\Theta(B_{\text{buf}}\cdot\text{sample size})$;  
extra time $\Theta(B_{\text{buf}}\cdot \pdim)$ due to rehearsal examples.

\paragraph{Architecture-based CL.}  
Memory grows with tasks: $\Theta(\sum_t \pdim_t)$;  
time increases with model expansion: $\Theta(\pdim_{\text{total}})$.

\paragraph{HTCL vs. standard CL.}  
HTCL trades moderate extra computation (bounded permutation search + curvature integration)  
for improved task-order robustness and multi-timescale consolidation, while maintaining  
near-linear practical complexity through diagonal or low-rank second-order approximations.

\subsection{Impact of Task Ordering on Model's Performance}
\label{app: task-order-impact}
In this section, we demonstrate the impact of task ordering on model performance across multiple domains, such as image classification, text classification, graph-based node classification, and synthetic sine regression tasks. Additionally, we examine the role of task order in facilitating positive transfer and mitigating interference.

\subsubsection{Illustrative Example: MNIST with BCL~\cite{raghavan2021formalizing}}
In this setup, a model sequentially learns five tasks in a BCL framework, each corresponding to different subsets of digits in the MNIST dataset. We analyze model performance under two distinct task sequences and log accuracy at each learning stage. Table~\ref{tab:mnist_results} presents task-wise accuracy values on corresponding test datasets for two task orders: Sequence A: $\langle 5, 1, 2, 3, 4 \rangle$ and Sequence B: $\langle1, 5, 2, 3, 4 \rangle$.

\begin{table}[htbp]
\caption{Task order impact on MNIST with BCL. The table reports task-wise accuracy at each stage of training. Variability in final accuracy across task orders indicates sensitivity to task order.}
\label{tab:mnist_results}
\resizebox{\textwidth}{!}{%
\begin{tabular}{cc|ccccc|c}
\textbf{} &
  \textbf{} &
  \multicolumn{5}{c|}{\textbf{Accuracy}} &
  \textbf{} \\ \hline
\rowcolor[HTML]{FFFFFF} 
\multicolumn{1}{c|}{\cellcolor[HTML]{FFFFFF}{\color[HTML]{000000} \textbf{Sequence}}} &
  {\color[HTML]{000000} \textbf{Learned Task}} &
  {\color[HTML]{000000} \textbf{Task\_1}} &
  {\color[HTML]{000000} \textbf{Task\_2}} &
  {\color[HTML]{000000} \textbf{Task\_3}} &
  {\color[HTML]{000000} \textbf{Task\_4}} &
  {\color[HTML]{000000} \textbf{Task\_5}} &
  {\color[HTML]{000000} \textbf{Average Accuracy}} \\ \hline
\rowcolor[HTML]{FFFFFF} 
\multicolumn{1}{c|}{\cellcolor[HTML]{FFFFFF}{\color[HTML]{000000} }} &
  {\color[HTML]{000000} 5} &
  {\color[HTML]{000000} 36.0} &
  {\color[HTML]{000000} 53.0} &
  {\color[HTML]{000000} 17.0} &
  {\color[HTML]{000000} 74.0} &
  {\color[HTML]{000000} 92.0} &
  {\color[HTML]{000000} 54.4} \\
\rowcolor[HTML]{FFFFFF} 
\multicolumn{1}{c|}{\cellcolor[HTML]{FFFFFF}{\color[HTML]{000000} }} &
  {\color[HTML]{000000} 1} &
  {\color[HTML]{000000} 100.0} &
  {\color[HTML]{000000} 53.0} &
  {\color[HTML]{000000} 17.0} &
  {\color[HTML]{000000} 79.0} &
  {\color[HTML]{000000} 87.0} &
  {\color[HTML]{000000} 67.2} \\
\rowcolor[HTML]{FFFFFF} 
\multicolumn{1}{c|}{\cellcolor[HTML]{FFFFFF}{\color[HTML]{000000} }} &
  {\color[HTML]{000000} 2} &
  {\color[HTML]{000000} 99.0} &
  {\color[HTML]{000000} 96.0} &
  {\color[HTML]{000000} 49.0} &
  {\color[HTML]{000000} 94.0} &
  {\color[HTML]{000000} 96.0} &
  {\color[HTML]{000000} 86.8} \\
\rowcolor[HTML]{FFFFFF} 
\multicolumn{1}{c|}{\cellcolor[HTML]{FFFFFF}{\color[HTML]{000000} }} &
  {\color[HTML]{000000} 3} &
  {\color[HTML]{000000} 96.0} &
  {\color[HTML]{000000} 90.0} &
  {\color[HTML]{000000} 93.0} &
  {\color[HTML]{000000} 87.0} &
  {\color[HTML]{000000} 89.0} &
  {\color[HTML]{000000} 91.0} \\
\rowcolor[HTML]{FFFFFF} 
\multicolumn{1}{c|}{\multirow{-5}{*}{\cellcolor[HTML]{FFFFFF}{\color[HTML]{000000} \textbf{A}}}} &
  {\color[HTML]{000000} \textbf{4}} &
  {\color[HTML]{000000} \textbf{97.0}} &
  {\color[HTML]{000000} \textbf{92.0}} &
  {\color[HTML]{000000} \textbf{88.0}} &
  {\color[HTML]{000000} \textbf{98.0}} &
  {\color[HTML]{000000} \textbf{95.0}} &
  {\color[HTML]{000000} \textbf{94.0}} \\ \hline
\rowcolor[HTML]{FFFFFF} 
\multicolumn{1}{c|}{\cellcolor[HTML]{FFFFFF}{\color[HTML]{000000} }} &
  {\color[HTML]{000000} 1} &
  {\color[HTML]{000000} 99.0} &
  {\color[HTML]{000000} 43.0} &
  {\color[HTML]{000000} 41.0} &
  {\color[HTML]{000000} 67.0} &
  {\color[HTML]{000000} 51.0} &
  {\color[HTML]{000000} 60.2} \\
\rowcolor[HTML]{FFFFFF} 
\multicolumn{1}{c|}{\cellcolor[HTML]{FFFFFF}{\color[HTML]{000000} }} &
  {\color[HTML]{000000} 5} &
  {\color[HTML]{000000} 98.0} &
  {\color[HTML]{000000} 53.0} &
  {\color[HTML]{000000} 11.0} &
  {\color[HTML]{000000} 75.0} &
  {\color[HTML]{000000} 94.0} &
  {\color[HTML]{000000} 66.2} \\
\rowcolor[HTML]{FFFFFF} 
\multicolumn{1}{c|}{\cellcolor[HTML]{FFFFFF}{\color[HTML]{000000} }} &
  {\color[HTML]{000000} 3} &
  {\color[HTML]{000000} 98.0} &
  {\color[HTML]{000000} 97.0} &
  {\color[HTML]{000000} 51.0} &
  {\color[HTML]{000000} 74.0} &
  {\color[HTML]{000000} 93.0} &
  {\color[HTML]{000000} 85.8} \\
\rowcolor[HTML]{FFFFFF} 
\multicolumn{1}{c|}{\cellcolor[HTML]{FFFFFF}{\color[HTML]{000000} }} &
  {\color[HTML]{000000} 3} &
  {\color[HTML]{000000} 96.0} &
  {\color[HTML]{000000} 95.0} &
  {\color[HTML]{000000} 94.0} &
  {\color[HTML]{000000} 63.0} &
  {\color[HTML]{000000} 45.0} &
  {\color[HTML]{000000} 78.6} \\
\rowcolor[HTML]{FFFFFF} 
\multicolumn{1}{c|}{\multirow{-5}{*}{\cellcolor[HTML]{FFFFFF}{\color[HTML]{000000} \textbf{B}}}} &
  {\color[HTML]{000000} \textbf{4}} &
  {\color[HTML]{000000} \textbf{92.0}} &
  {\color[HTML]{000000} \textbf{90.0}} &
  {\color[HTML]{000000} \textbf{86.0}} &
  {\color[HTML]{000000} \textbf{99.0}} &
  {\color[HTML]{000000} \textbf{91.0}} &
  {\color[HTML]{000000} \textbf{91.6}}
\end{tabular}%
}
\end{table}

To construct this table, we sequentially present tasks in a specified order, training the model on one task at a time. After training each task, we evaluate the model's performance on all tasks (previously seen or unseen) and record the accuracy. The accuracy values in the table represent the model's ability to retain knowledge as it progresses through the task sequence. The last column, labeled Average Accuracy, represents the model's overall performance across all tasks at each stage of training.

As training progresses, learning a new task typically leads to an immediate increase in accuracy for that specific task. However, previously learned tasks may experience either a stability effect (no change), positive transfer (accuracy improvement), or interference (accuracy degradation). The table reflects this effect, where the accuracy values evolve as more tasks are introduced. When the model completes training on Task 4, it has been exposed to all tasks, and its performance stabilizes across them. The lower half of the table follows the same calculation process for Sequence B, demonstrating how a different task order can result in varied accuracy trends. A comparison of task-wise accuracy and overall average accuracy between Sequence A and Sequence B clearly reveals variability in final performance, reinforcing the impact of task ordering on CL.

Since Task 4 appears at the end of both sequences, we expect its accuracy to be relatively high, as the model has just learned it. This is evident in rows 5 and 10 of Table~\ref{tab:mnist_results}. However, accuracy for other tasks varies significantly depending on their position in the sequence, demonstrating how order influences retention. Comparing Sequence A and Sequence B, we observe two key trends:

\begin{itemize}
    \item The accuracy of each task after it has been trained is shown in columns 3-7 and indicate a significant variation in accuracy~(approximately 2-5\%).
    \item The final average accuracy calculated after learning the final task differs between sequences (94\% and 91.6\%), suggesting that task order influences the performance.
\end{itemize}

This illustrative example demonstrates that task ordering affects both retention and adaptation. In subsequent sections, we extend this analysis across multiple datasets and learning paradigms.

\subsubsection{Impact of Task Number on Forgetting}
\label{app:impact-example}

\begin{figure}[htbp]
    \centering
    \includegraphics[scale=0.1]{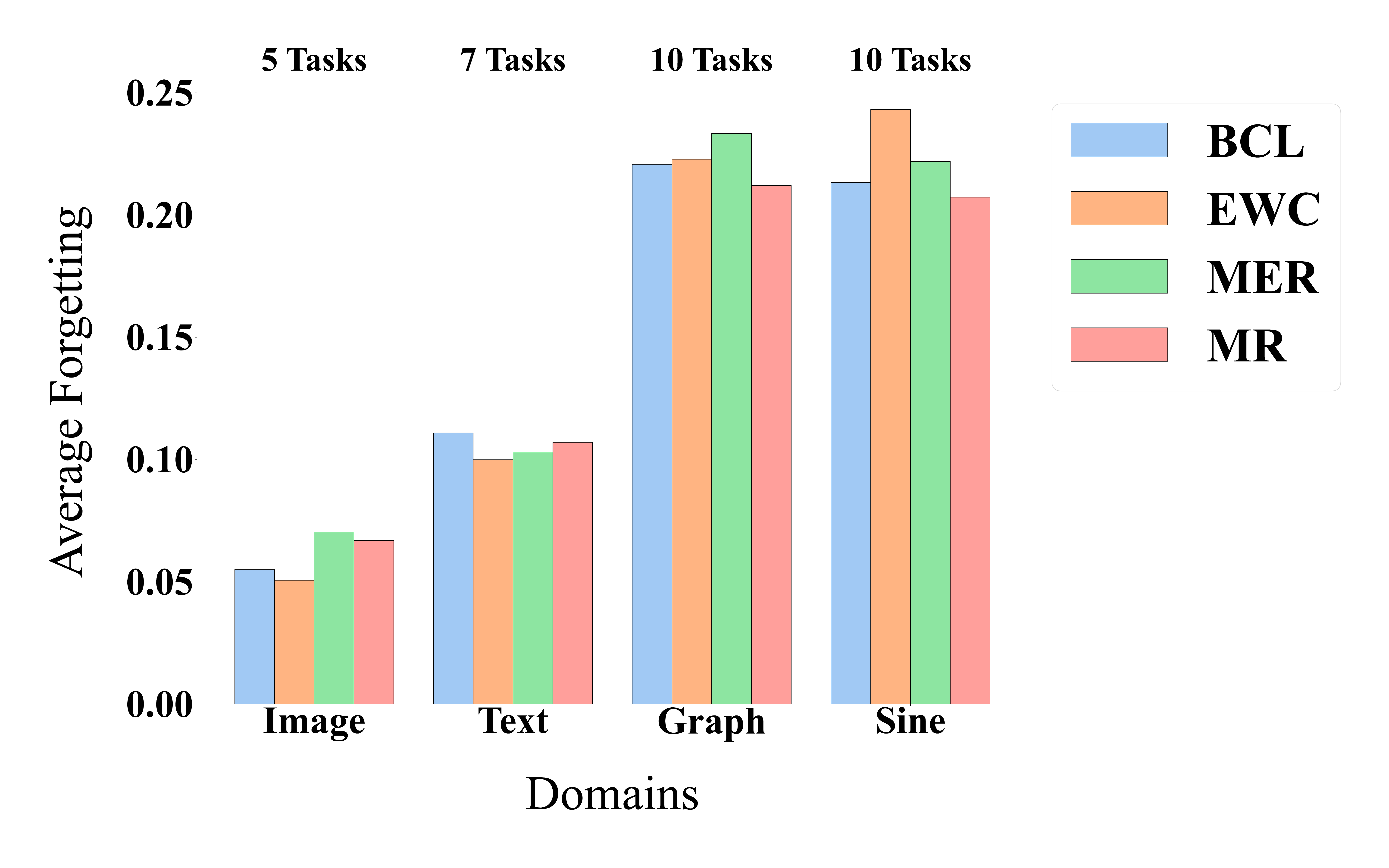}
    \caption{Impact of the number of tasks on forgetting in CL across different domains. Forgetting increases as the number of tasks grows, with models experiencing up to 20\% accuracy drop in a 10-task setting.}
    \label{fig:forgetting_vs_tasks}
\end{figure}

Figure~\ref{fig:forgetting_vs_tasks} illustrates how forgetting exacerbates as the number of tasks increases. As we move left to right of the plot, the tasks increase. Corresponding to an increase in the number of tasks, the height of the bar plots~(which indicates the extent of forgetting) also increases. Moreover, this is true across multiple domains such as text, graph, image and numeric. This observation highlights the persistent challenge of knowledge retention in CL.

The trend depicted in Figure~\ref{fig:forgetting_vs_tasks} correlates with our hypothesis. That is, as the number of tasks increase, the likelihood of encountering suboptimal task sequences grows, leading to:
\begin{itemize}
    \item \textbf{Greater Forgetting}: Some sequences introduce tasks with highly dissimilar representations, accelerating catastrophic forgetting.
    \item \textbf{Stronger Interference}: Certain task orders disrupt previously learned representations, causing degradation in overall model performance.
\end{itemize}

To further investigate this phenomenon, we analyze model performance across 50 distinct task sequences, measuring accuracy/mse trends and loss variations. The box plots corresponding to this analysis is plotted in Fig. \ref{fig:results}. We do this analysis across four methods and four domains. 

\begin{figure}[t]
    \centering
    \resizebox{\columnwidth}{!}{%
    \begin{tabular}{p{0.12\columnwidth} p{0.22\columnwidth} p{0.22\columnwidth} p{0.22\columnwidth} p{0.22\columnwidth}}
        & \textbf{Image} & \textbf{Text} & \textbf{Graph} & \textbf{Sine} \\
        \textbf{EWC} &
        \includegraphics[width=\linewidth]{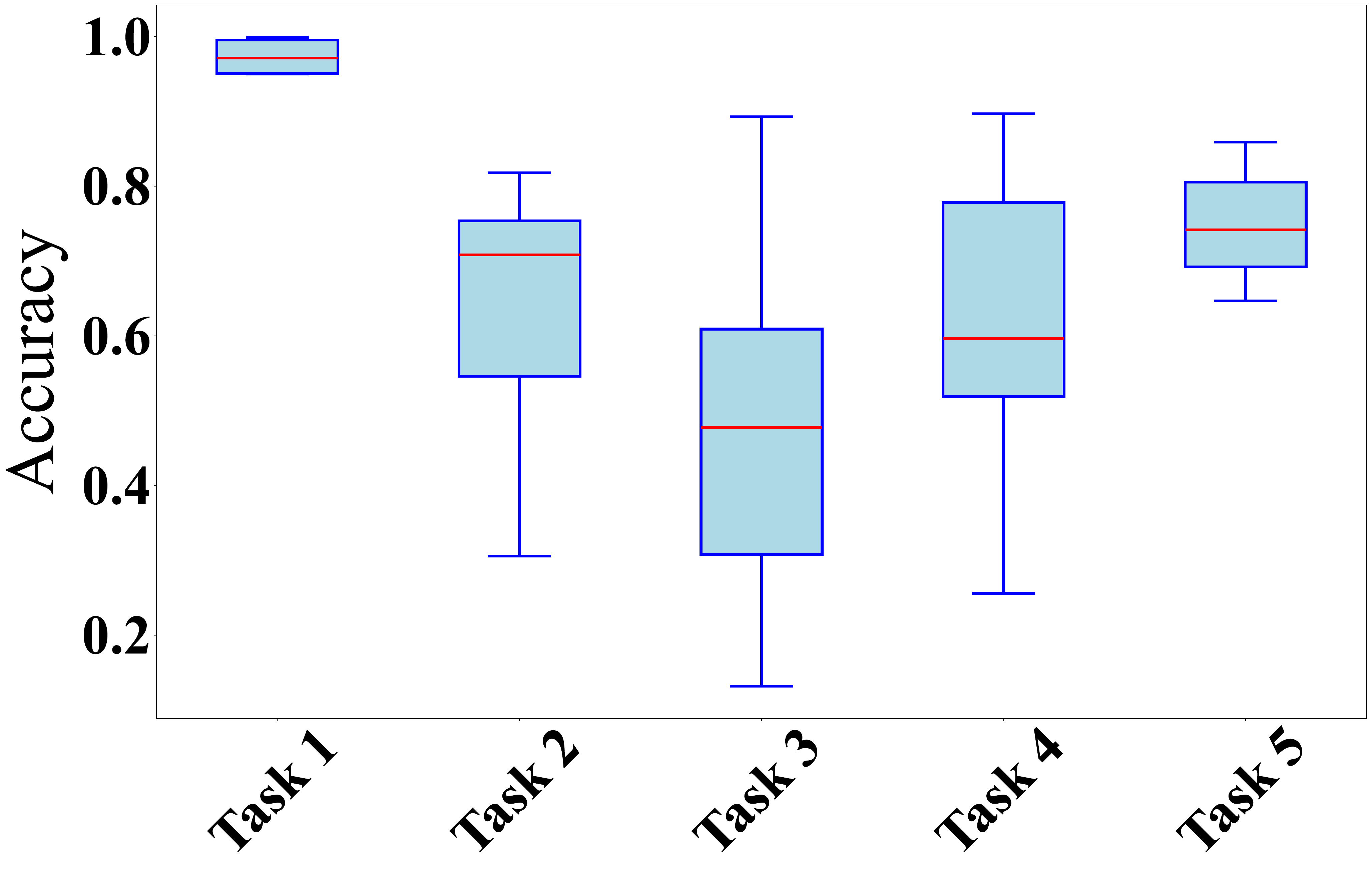} &
        \includegraphics[width=\linewidth]{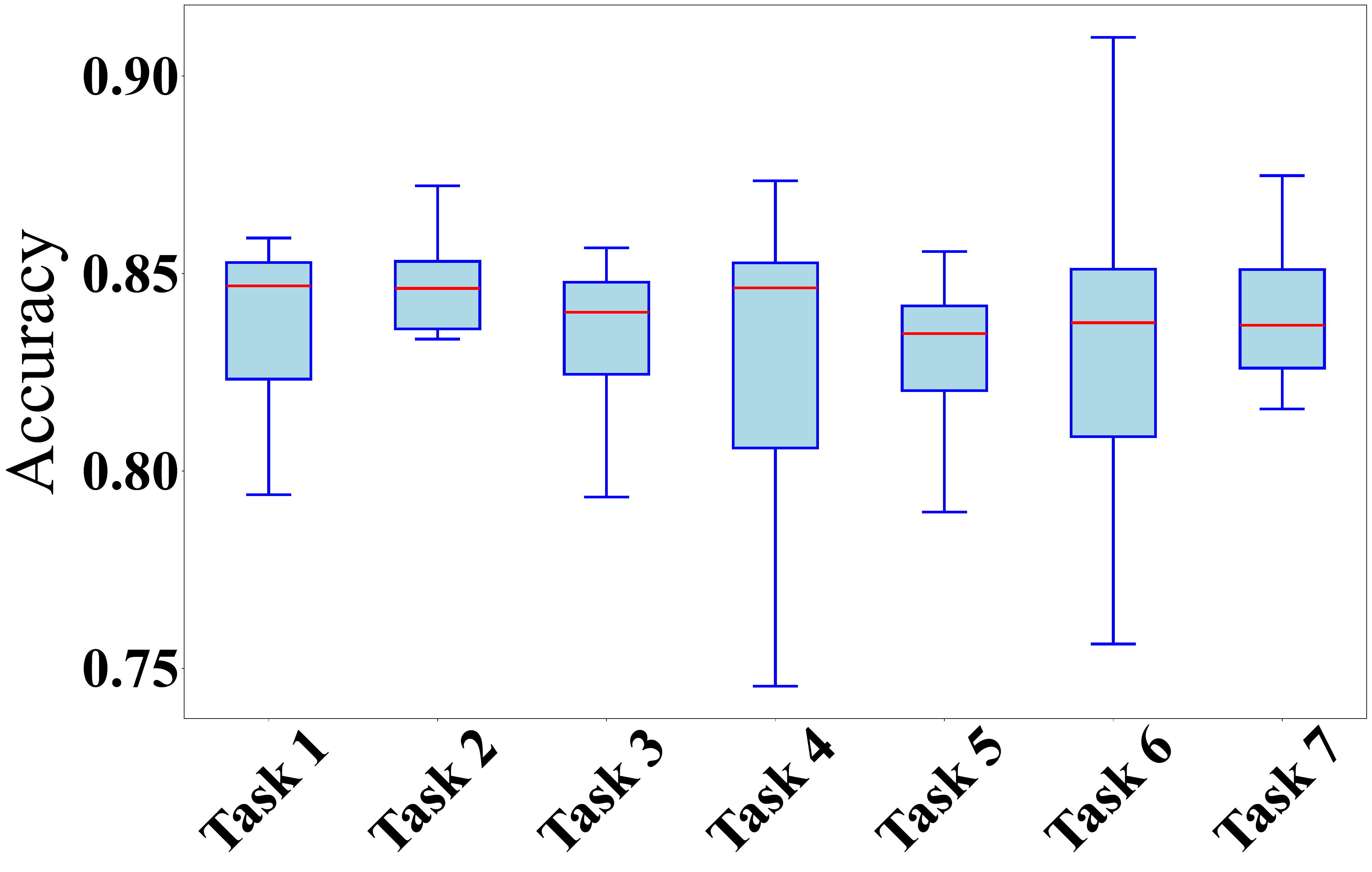} &
        \includegraphics[width=\linewidth]{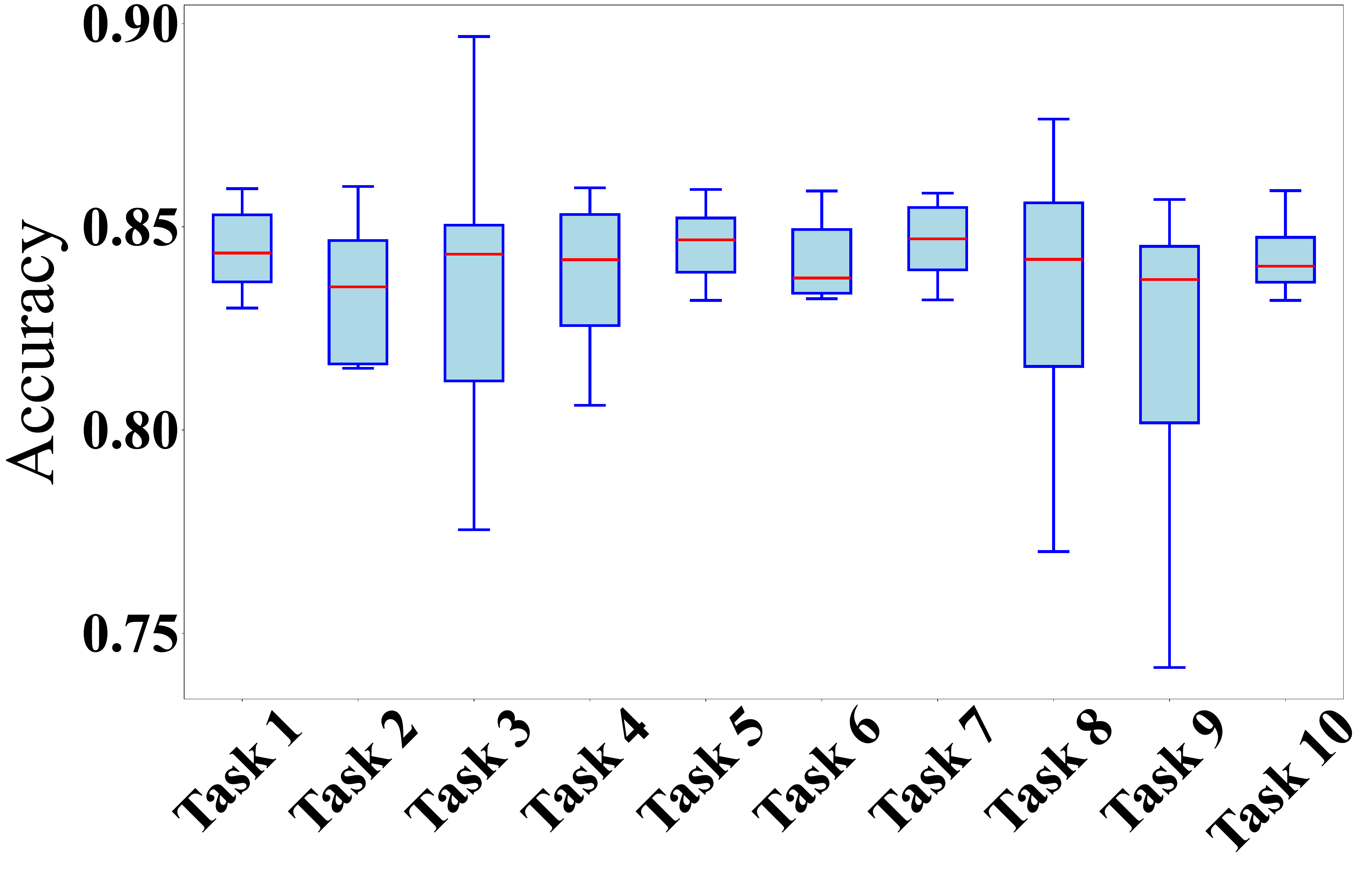} &
        \includegraphics[width=\linewidth]{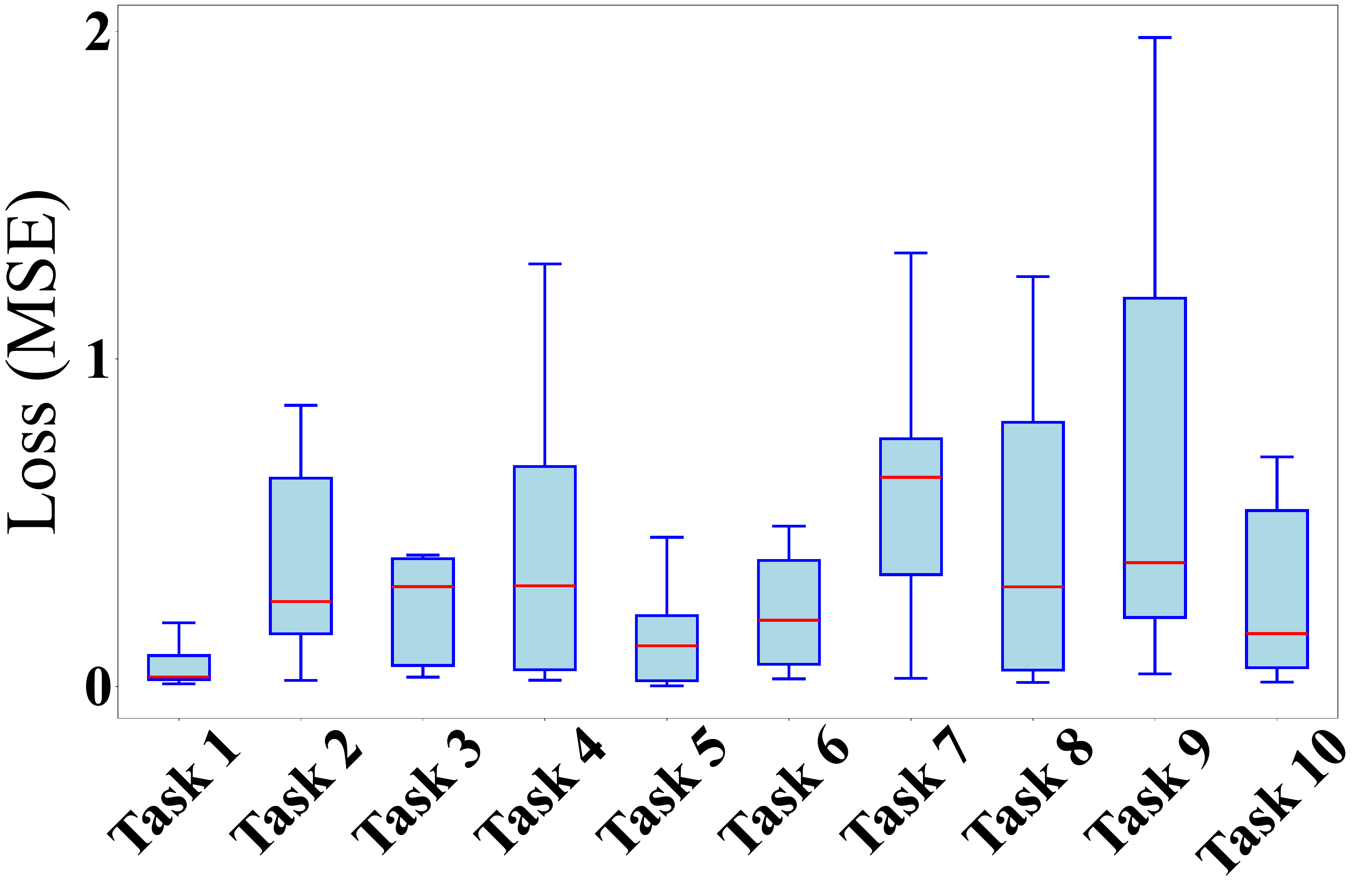} \\
        \textbf{ER} &
        \includegraphics[width=\linewidth]{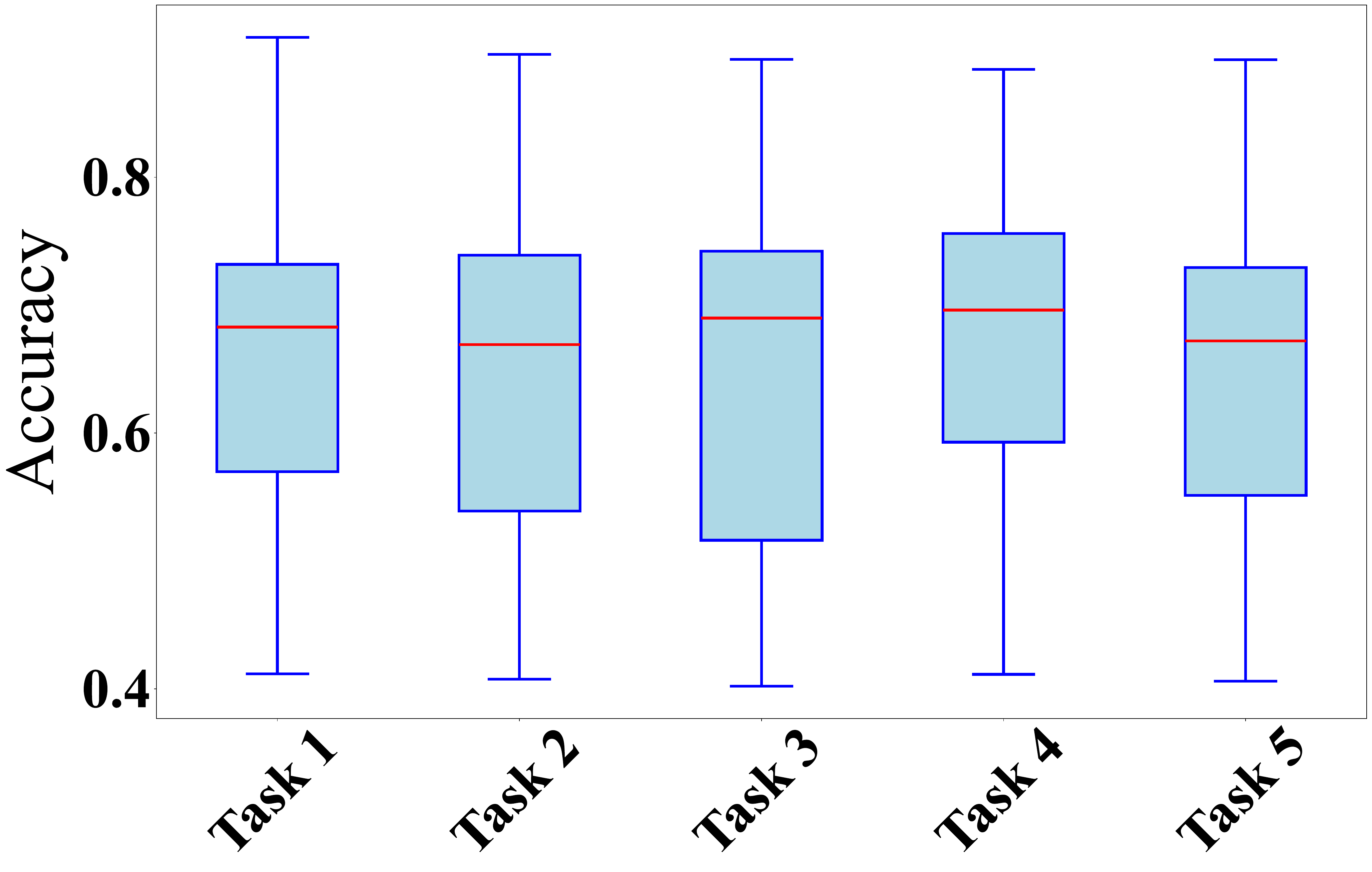} &
        \includegraphics[width=\linewidth]{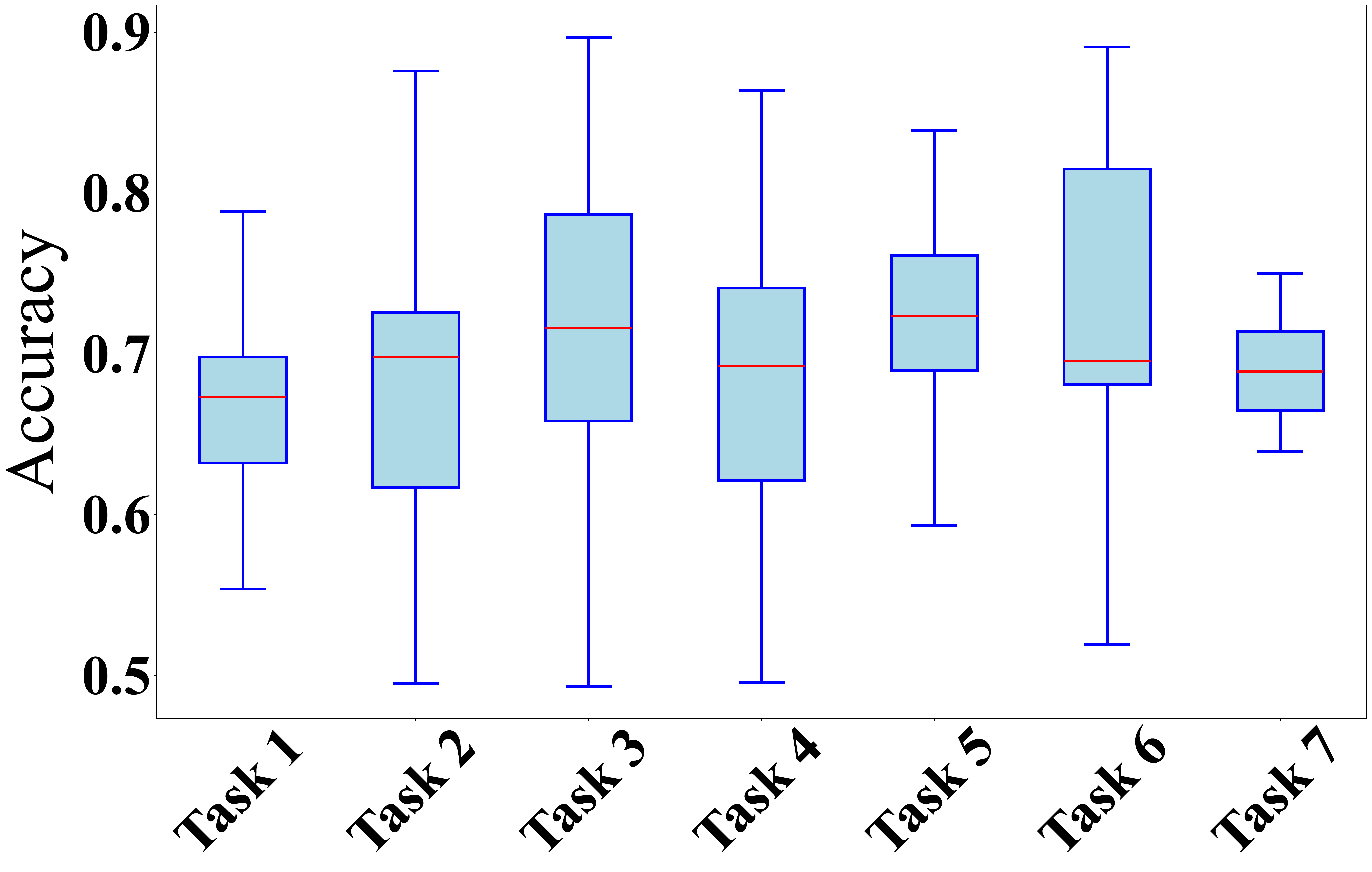} &
        \includegraphics[width=\linewidth]{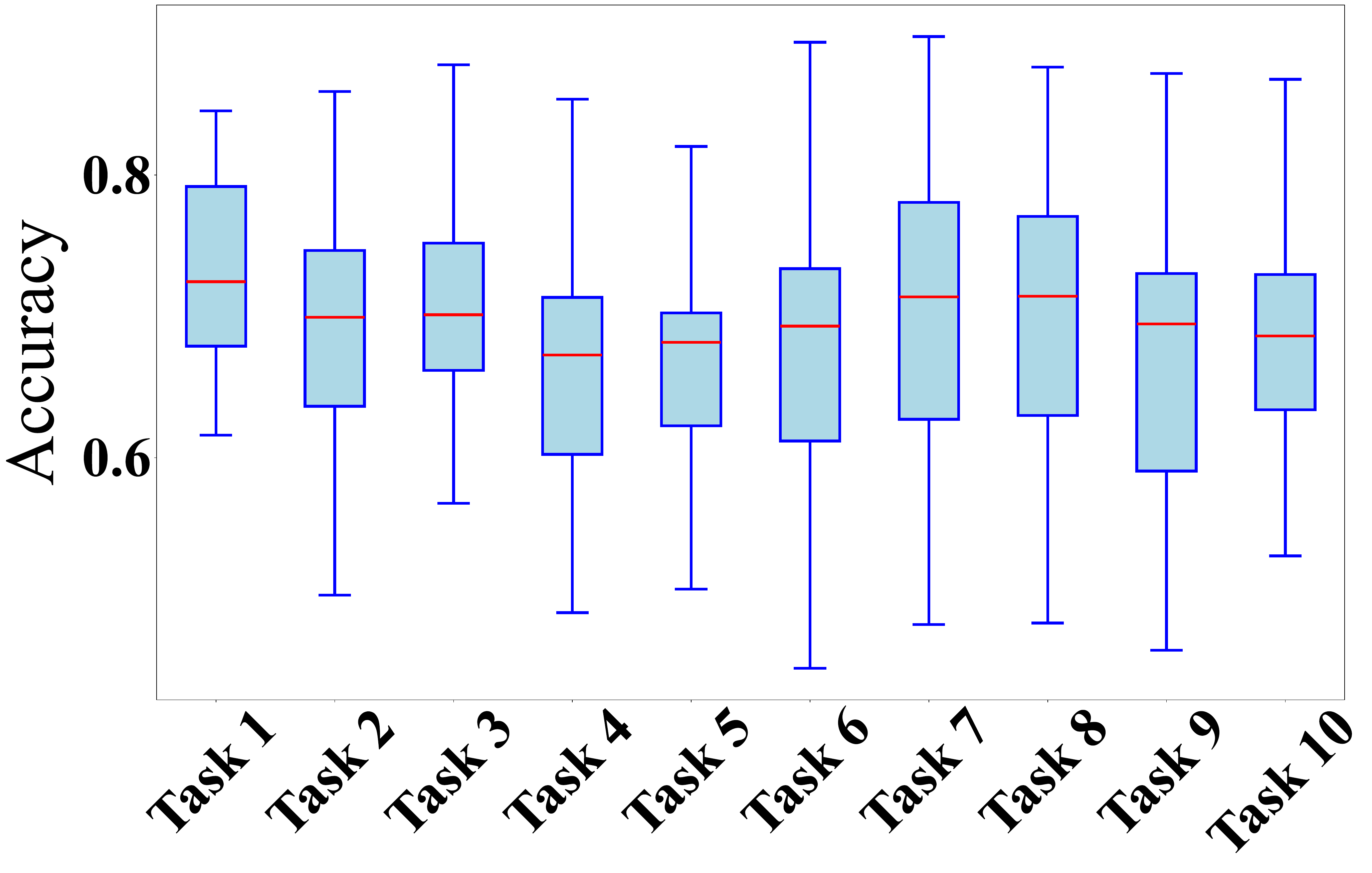} &
        \includegraphics[width=\linewidth]{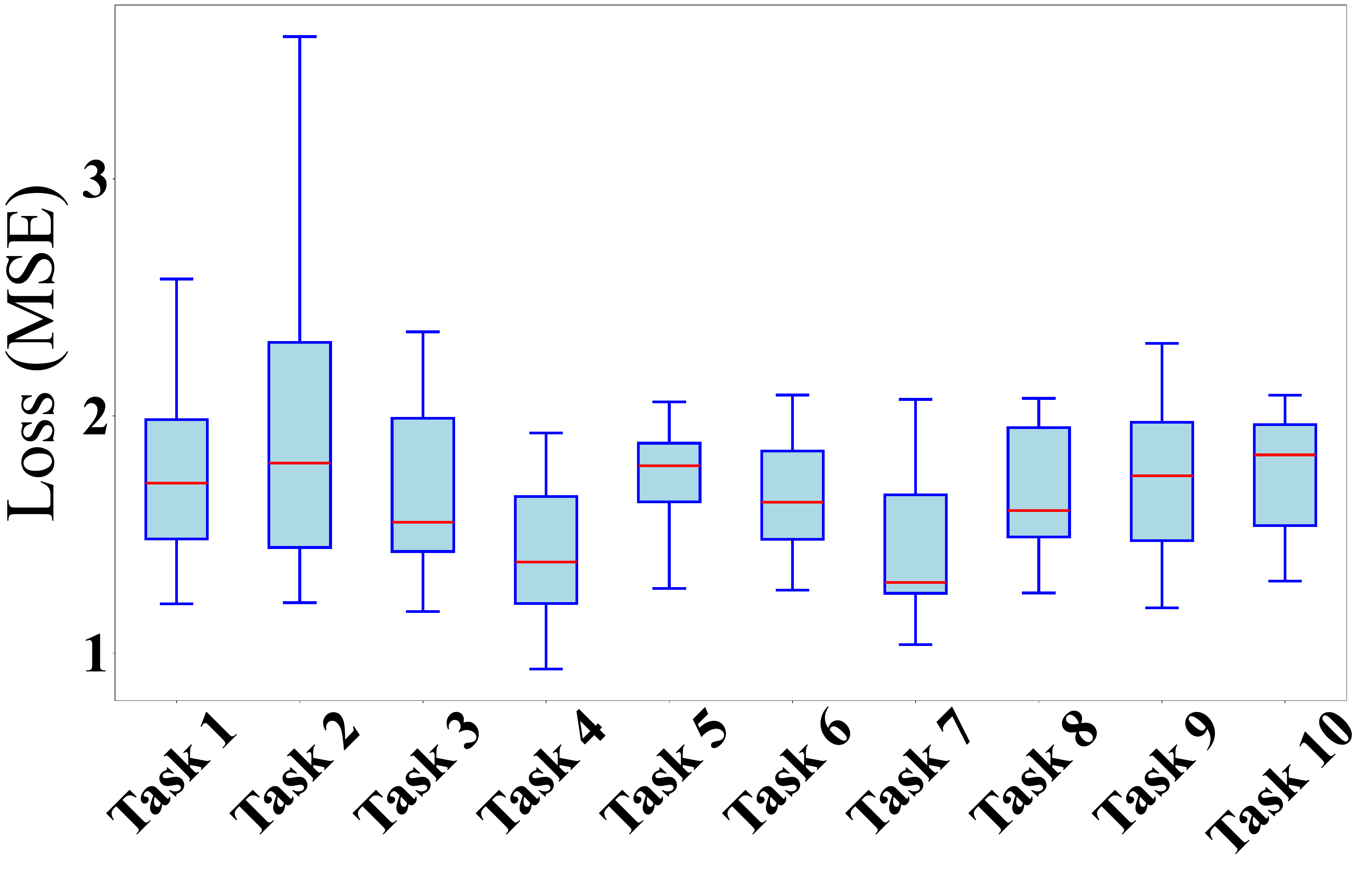} \\
        \textbf{BCL} &
        \includegraphics[width=\linewidth]{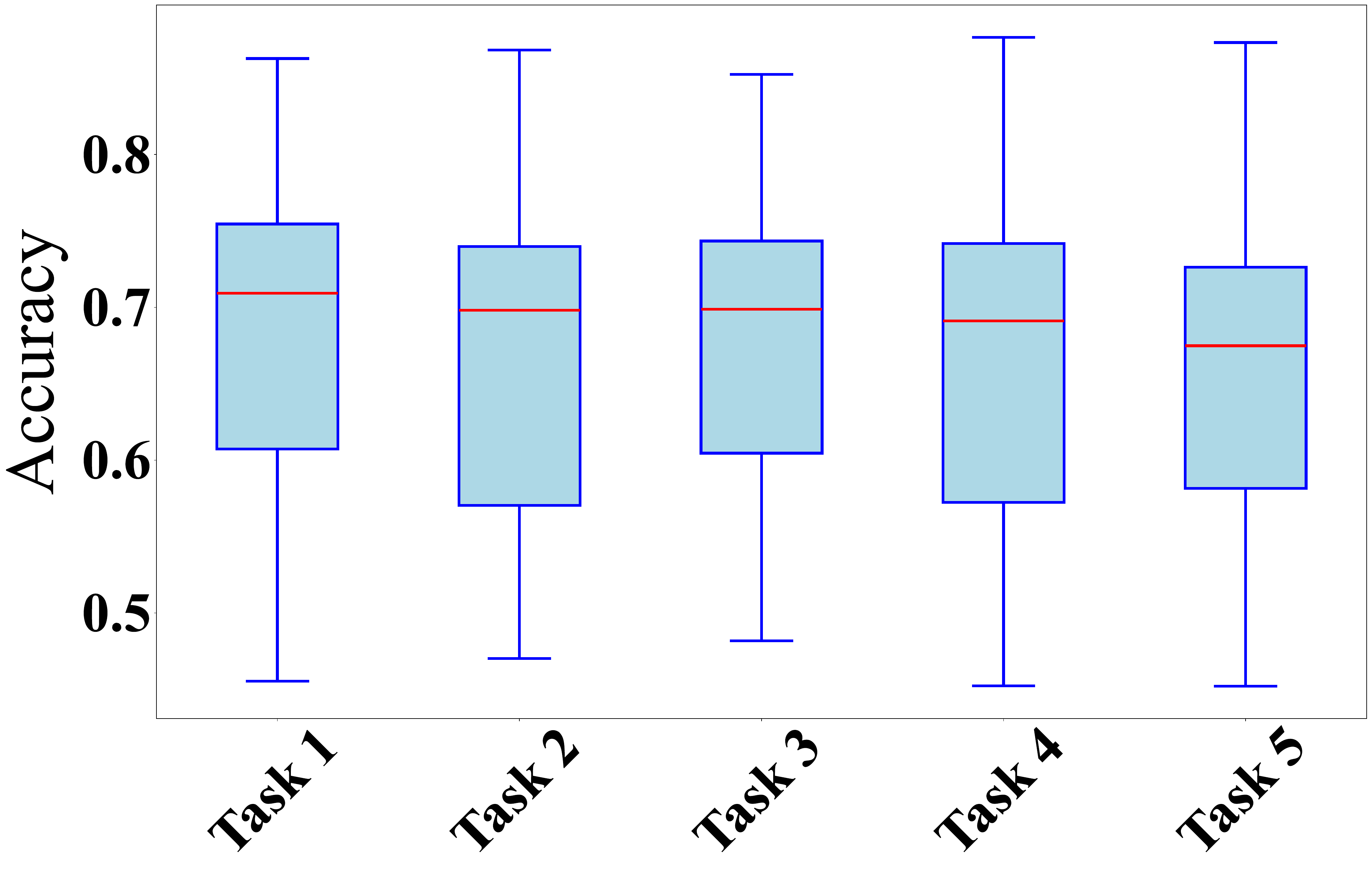} &
        \includegraphics[width=\linewidth]{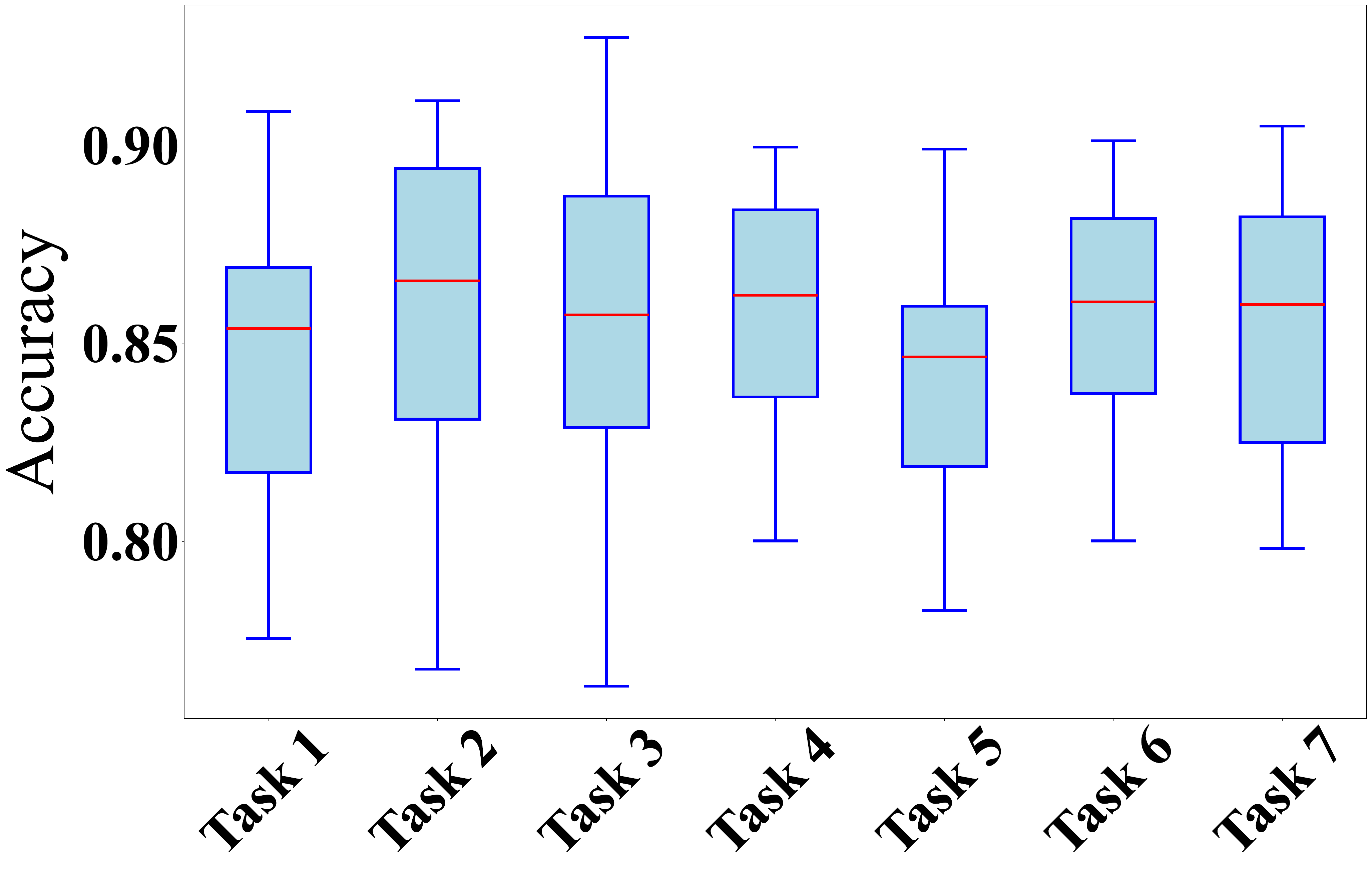} &
        \includegraphics[width=\linewidth]{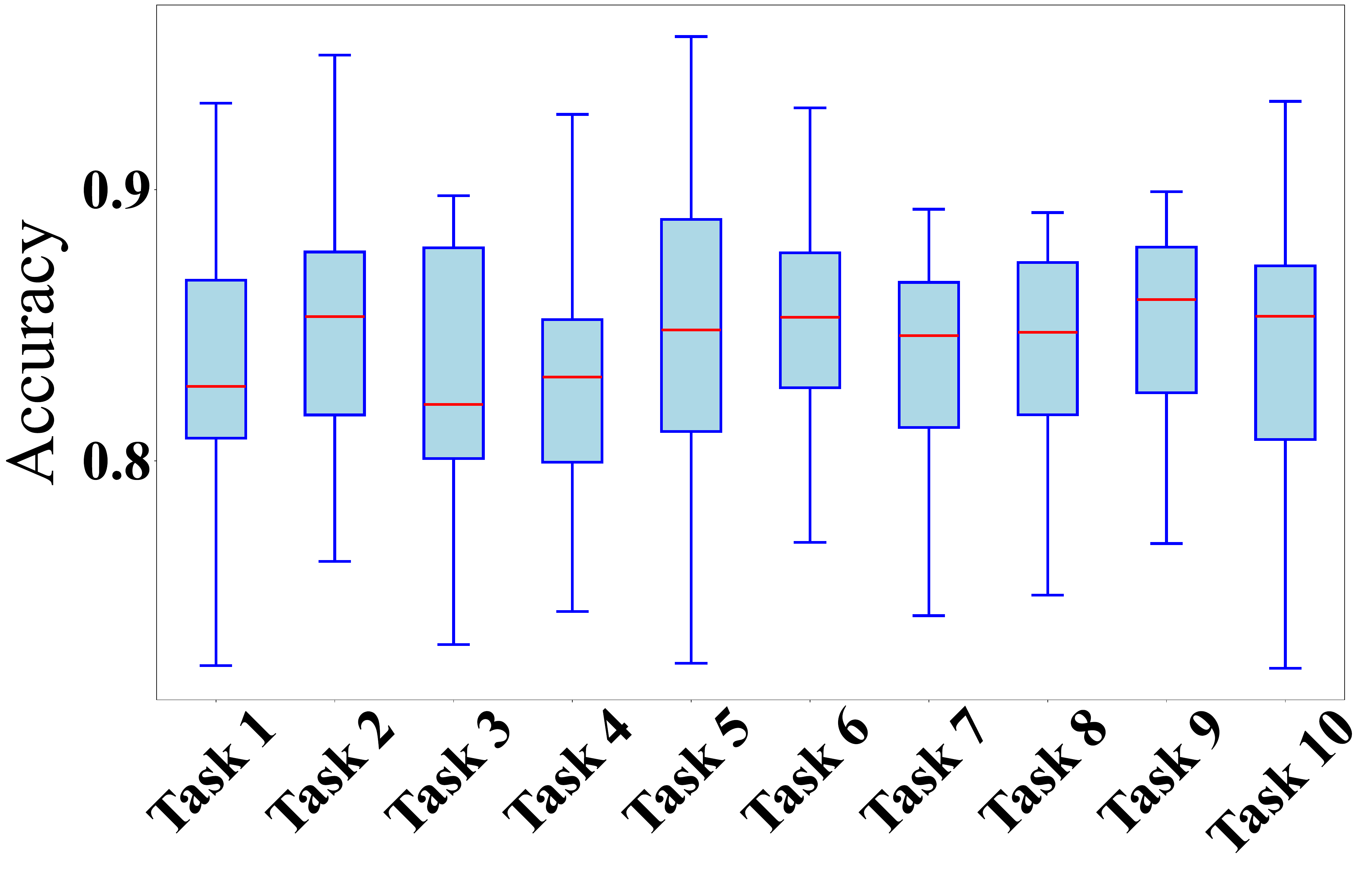} &
        \includegraphics[width=\linewidth]{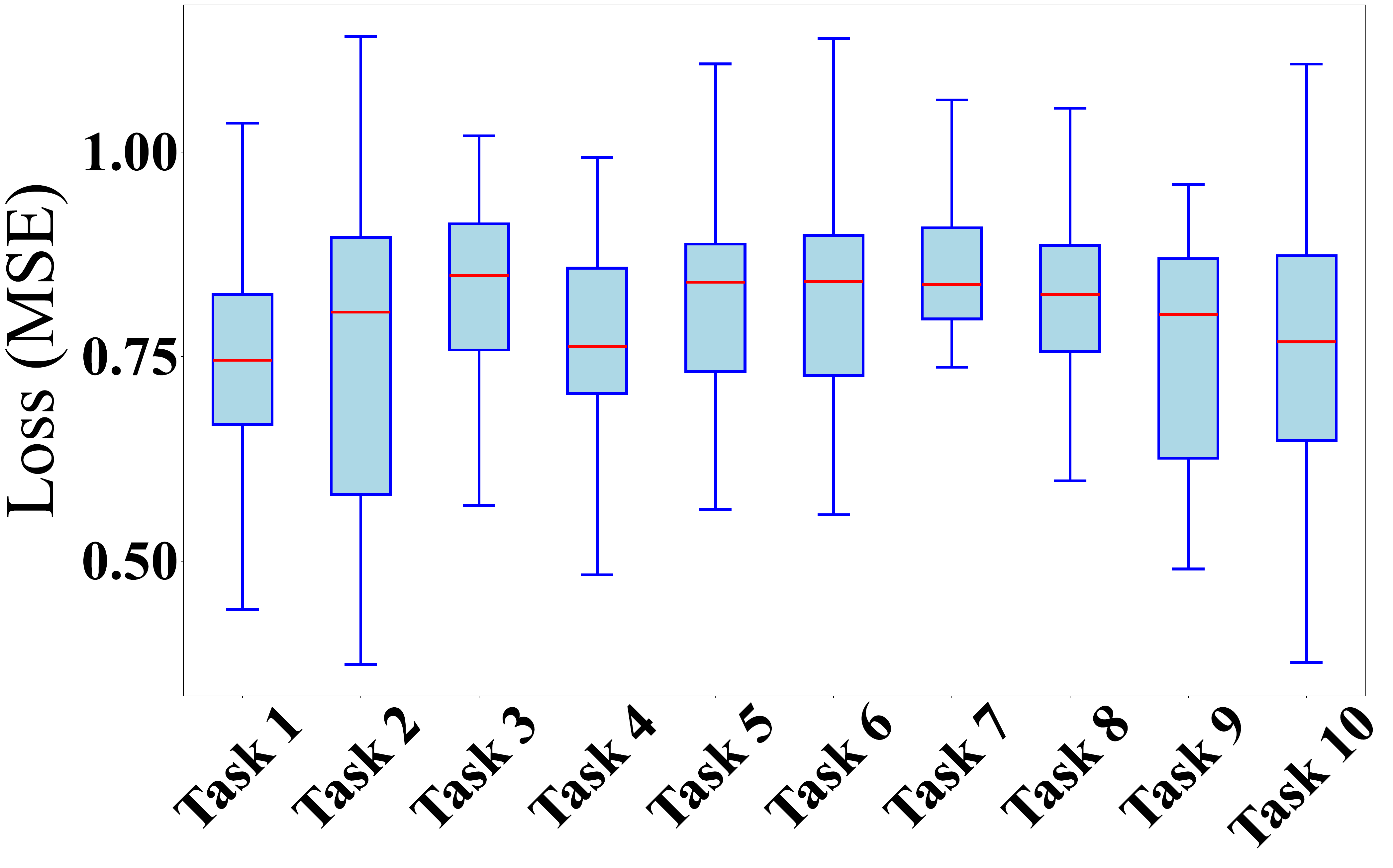} \\
        \textbf{MER} &
        \includegraphics[width=\linewidth]{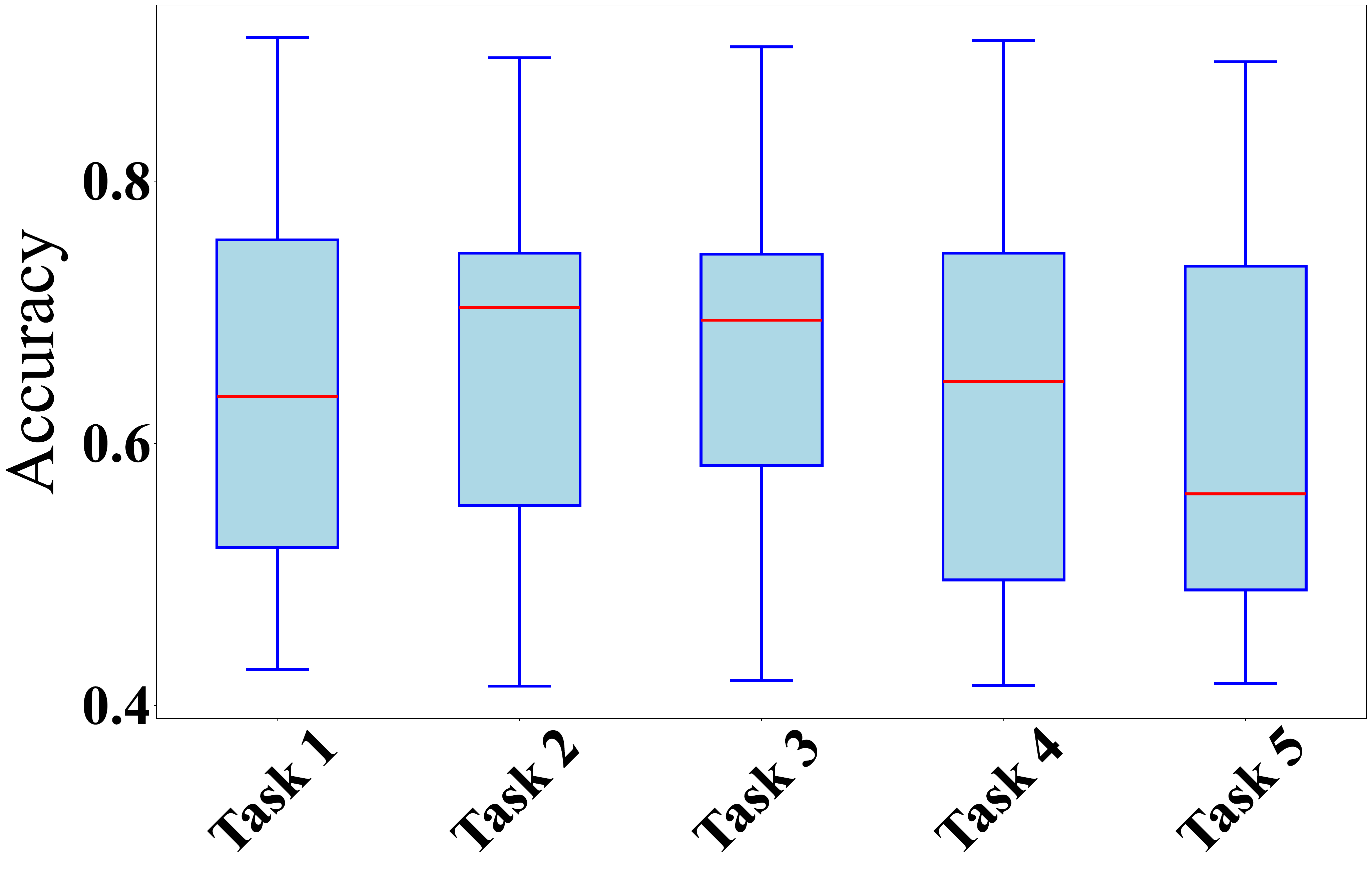} &
        \includegraphics[width=\linewidth]{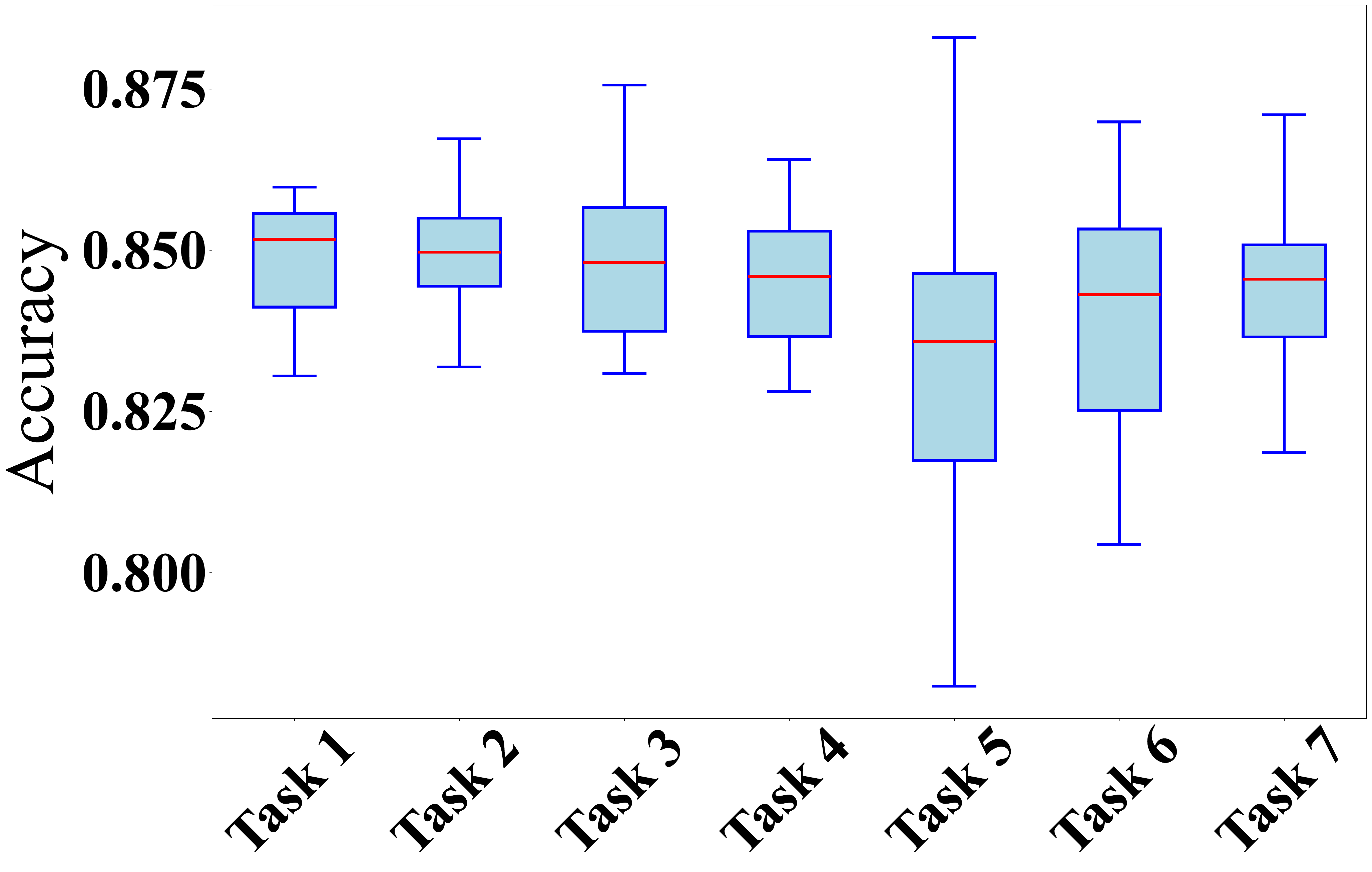} &
        \includegraphics[width=\linewidth]{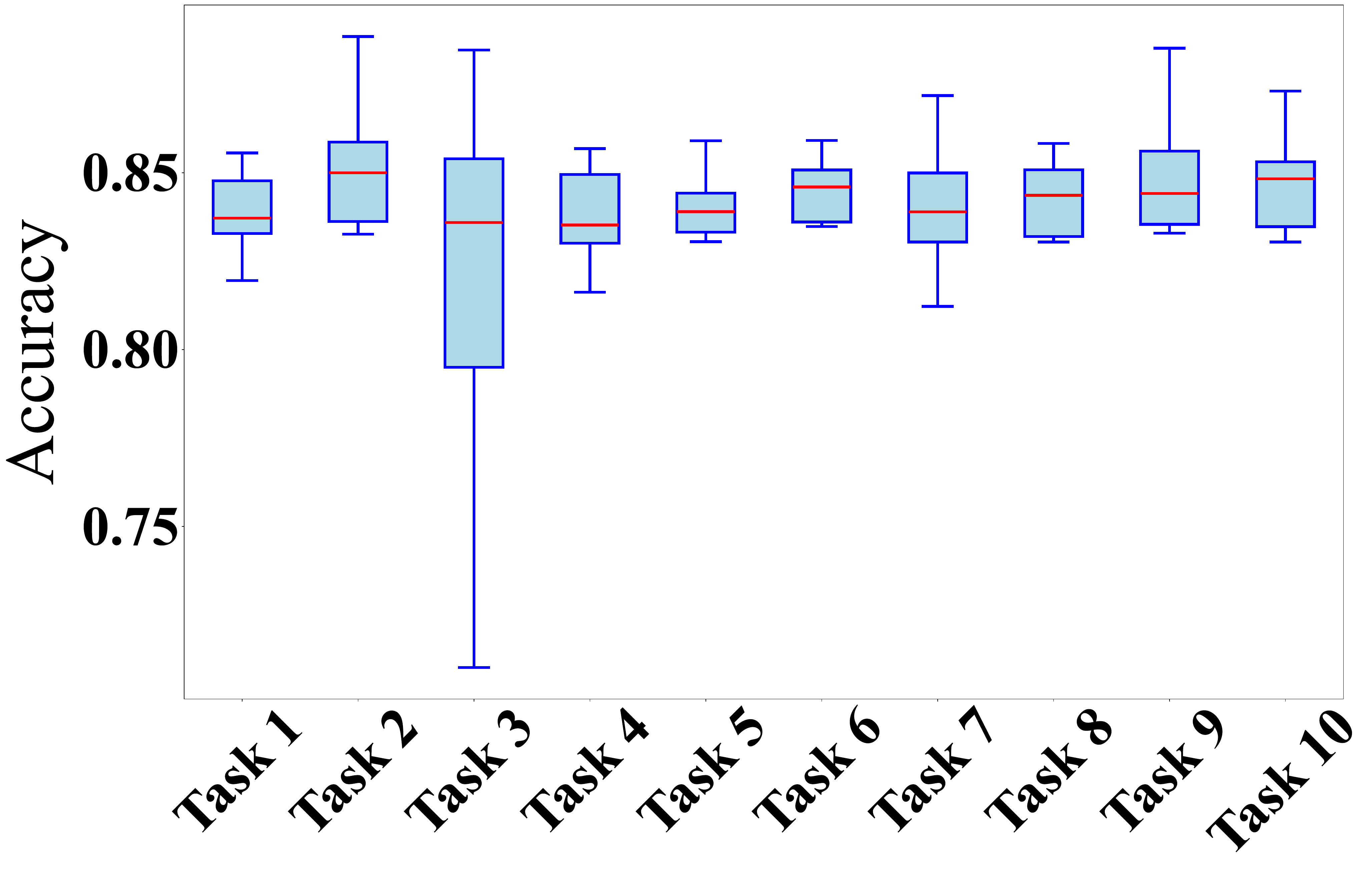} &
        \includegraphics[width=\linewidth]{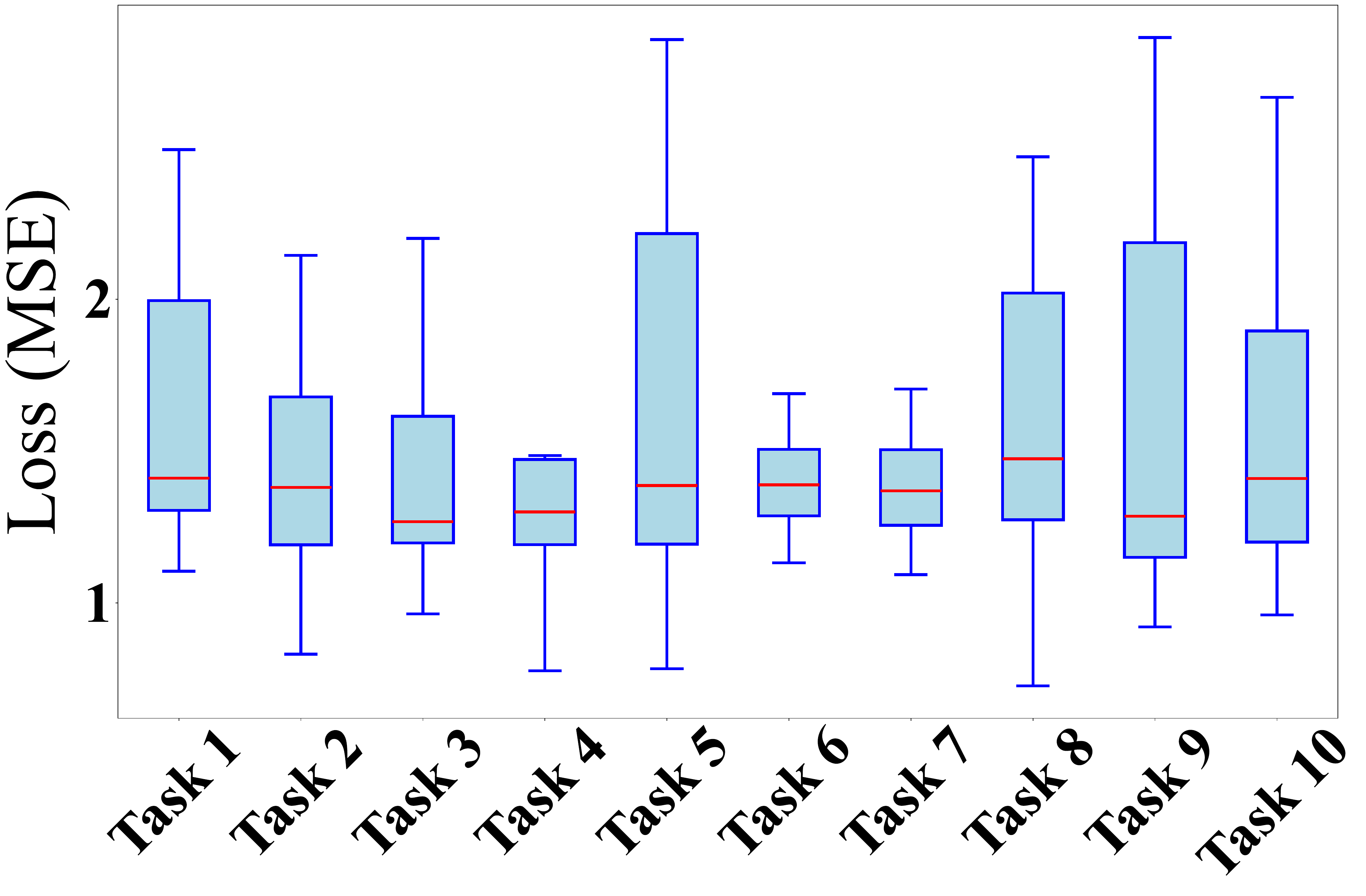} \\
    \end{tabular}%
    }
    \caption{Impact of task order comparison of CL methods—EWC, MER, BCL, and ER—across diverse domains, including image classification, text-based classification tasks, synthetic graph-structured data, and synthetic sine wave regression. The accuracies and mean squared error (MSE) losses shown are calculated after the model has learned all the tasks. Random task orders were generated for each method and domain, and the accuracy or loss was logged for each task in the sequence. The high variability observed in the box plots highlights the significant impact of task order on model performance in CL. This phenomenon is consistent across all evaluated domains.}
    \label{fig:results}
\end{figure}

\subsubsection{Impact of Task Order Across Domains}
The box plots illustrate the variability in accuracy (or MSE loss) for each task across different sequences, highlighting the significant sensitivity of models to the task order.

To generate this plot, we randomly created 50 task sequences for each domain and logged the models' performance at each training stage. By training stage, we refer to the point at which a particular task is introduced to the model. While intermediate accuracy or loss values exhibited expected fluctuations, the final performance (after training on all tasks) demonstrates a striking dependency on the task order. To observe this, observe the size of the box plot corresponding to each task in each subplot in Fig.~\ref{fig:results}. Here,  the large interquartile range (IQR) underscores the extent of this variability that we hypothesize.

\textbf{Variability in Performance across Task Orders:}
Figure~\ref{fig:results} highlights the varying sensitivity of different CL methods to task ordering. This sensitivity is evident in the spread of the box plots, where methods with wider distributions exhibit higher variability in final accuracy (or loss), whereas methods with tighter distributions demonstrate more robustness to task order. 

We also note that the performance metrics can be disrupted by inconsistent sequencing. The IQR observed across methods and domains in Fig.~\ref{fig:results} reaffirms that task order is not merely a peripheral concern but a fundamental challenge in CL. Most approaches across different datasets struggle with sequence dependencies, suggesting the need for explicit task ordering strategies to enhance stability. While the range of variation is method-dependent, the range is extensive, with the minimum across all domains and methods being 2\% and the maximum being 22\%. Addressing this challenge requires domain-specific strategies to ensure consistency in learning outcomes. These insights further motivate the need for optimizing task sequencing strategies, as explored in the subsequent section, where we formalize the task ordering problem as an optimization challenge.

\subsection{Extended Results}
\label{app:extended-results}

\subsubsection{Code}
Code will be released upon acceptance of this paper. 

\subsubsection{Evaluation Metrics}
\label{sec:metrics}

To rigorously assess the performance and stability of our framework, we employ three standard metrics commonly used in the CL literature. Let $A_{i, j}$ denote the test accuracy on task $j$ after the model has finished learning task $i$, and let $\ntasks$ be the total number of tasks. 

\textbf{Mean Accuracy.} We report the mean accuracy across all observed tasks after the training sequence is complete . It is defined as: 

\begin{equation}
    \text{Mean Accuracy} = \frac{1}{\ntasks} \sum_{j=1}^{\ntasks} A_{\ntasks, j}
\end{equation}

\textbf{Average Forgetting.} To quantify the extent of catastrophic forgetting, we measure the average decrease in performance for each task from its peak accuracy (immediately after training on that task) to its final accuracy \cite{yoon2019scalable}. This is calculated as: 

\begin{equation}
    \text{Avg. Forgetting} = \frac{1}{\ntasks-1} \sum_{j=1}^{\ntasks-1} \max_{l \in \{1, \dots, \ntasks-1\}} (A_{l,j} - A_{\ntasks,j})
\end{equation}

\textbf{Standard Deviation.} To evaluate robustness against random task orderings, we report the standard deviation of the final average accuracy across all evaluated permutations ~\cite{bell2022effect}. For a set of $P$ permutations, this is given by:

\begin{equation}
    \sigma_{\text{perm}} = \sqrt{\frac{1}{P} \sum_{p=1}^{P} (\text{ACC}_p - \mu_{\text{ACC}})^2}
\end{equation}

where $\text{ACC}_p$ is the Average Accuracy for the $p$-th permutation and $\mu_{\text{ACC}}$ is the mean over all permutations. A lower $\sigma_{\text{perm}}$ indicates a method that is less sensitive to the specific order in which tasks arrive.

\subsubsection{Detailed Setup}
\label{sec:detailed_setup}

We evaluate our approach on four diverse continual learning benchmarks spanning image classification, text classification, and graph node classification. Each dataset is partitioned into disjoint tasks following the domain-incremental learning protocol, where each task introduces a new subset of classes not seen in previous tasks. Table~\ref{tab:dataset_setup} summarizes the dataset configurations.

\begin{table}[!htbp]
\centering
\caption{Dataset configurations for continual learning experiments. Each dataset is split into sequential tasks with disjoint class subsets. Labels are remapped to $\{0, \ldots, C_t-1\}$ within each task.}
\label{tab:dataset_setup}
\resizebox{\textwidth}{!}{%
\begin{tabular}{lcccccc}
\toprule
\textbf{Dataset} & \textbf{Total Classes} & \textbf{Tasks} & \textbf{Classes/Task} & \textbf{Input Dim.} & \textbf{Buffer Size} & \textbf{Domain} \\
\midrule
Split-MNIST      & 10  & 5  & 2  & $28 \times 28 \times 1$   & 50  & Image \\
CIFAR-100        & 100 & 10 & 10 & $32 \times 32 \times 3$   & 500 & Image \\
20 Newsgroups    & 20  & 5  & 4  & 20,000 (TF-IDF)           & 100 & Text  \\
Cora             & 7   & 3  & 2--3 & 1,433 (node features)   & 50  & Graph \\
\bottomrule
\end{tabular}%
}
\end{table}

\paragraph{Image Classification.}
For \textbf{Split-MNIST}, we partition the 10 digit classes into 5 consecutive tasks, each containing 2 classes (e.g., Task 1: digits 0--1, Task 2: digits 2--3, etc.). The grayscale images ($28 \times 28$) are normalized and fed directly to the network. For \textbf{CIFAR-100}, we divide the 100 fine-grained classes into 10 tasks of 10 classes each. We apply standard data augmentation including random crops with 4-pixel padding and horizontal flips during training.

\paragraph{Text Classification.}
The \textbf{20 Newsgroups} dataset consists of approximately 18,000 newsgroup posts across 20 categories. We extract TF-IDF features with a vocabulary size of 20,000 and partition the categories into 5 tasks of 4 classes each.

\paragraph{Graph Classification.}
For \textbf{Cora}, a citation network with 2,708 nodes and 7 classes, we split the classes into 3 tasks. Nodes are randomly partitioned into 80\% training and 20\% test sets while preserving class distributions. We use the 1,433-dimensional bag-of-words node features as input.

\paragraph{Training Protocol.}
All experiments use SGD with momentum 0.9. Table~\ref{tab:hyperparams} details the hyperparameter settings. For replay-based methods (ER, DER, SER, and HTCL), we maintain an experience replay buffer that stores samples from previous tasks. The buffer is populated using reservoir sampling to ensure balanced representation across all seen classes.

\begin{table}[ht]
\centering
\caption{Training hyperparameters for each dataset.}
\label{tab:hyperparams}
\begin{tabular}{lcccc}
\toprule
\textbf{Parameter} & \textbf{Split-MNIST} & \textbf{CIFAR-100} & \textbf{20 Newsgroups} & \textbf{Cora} \\
\midrule
Learning Rate      & 0.01  & 0.1   & 0.001 & 0.01  \\
Batch Size         & 64    & 32    & 32    & 32    \\
Epochs per Task    & 10    & 10    & 10    & 10    \\
Weight Decay       & $10^{-4}$ & $10^{-4}$ & $10^{-4}$ & $10^{-4}$ \\
Buffer Size        & 50    & 500   & 100   & 50    \\
\bottomrule
\end{tabular}
\end{table}

\paragraph{HTCL-Specific Parameters.}
For our hierarchical approach, we evaluate configurations with $\nlev \in \{2, 3, 4, 5\}$ hierarchy levels. The group size is set to $\gsize=2$ tasks per group. We use Taylor update step size $\eta = 0.9$ with gradient clipping at $\|\cdot\|_{\max} = 1.0$ and regularization strength $\reg > -\mu_{\min}(\mathbf{H}^{(t-1)})$. The global model catch-up mechanism is enabled with 2 additional Taylor update iterations after the last task group.

All experiments are conducted on a single NVIDIA GPU with CUDA 12.3. We report results averaged over 200 random task orderings to account for permutation sensitivity, with standard deviations included. The random seed is fixed at 42 for reproducibility.

\subsubsection{Performance with Smaller Buffer Size}
\label{app:small-buffer}

Replay-based CL methods rely heavily on buffer capacity. When replay memory is sufficiently large, rehearsal can approximate joint training and substantially mitigate forgetting. However, in realistic low-memory regimes, replay becomes much less reliable. To evaluate the robustness of HTCL under such constraints, we conduct experiments on both SplitMNIST and Split CIFAR-100 while varying the replay buffer size. For SplitMNIST, we sweep buffer sizes in $\{30, 50, 100, 200, 500\}$, whereas for Split CIFAR-100 we evaluate buffers of size $50$, $200$, and $1000$ to reflect the substantially larger number of classes.

Figure~\ref{fig:memory} reports the performance of ER, DER, and SER, together with their HTCL-augmented variants. Under very small buffers, standalone replay baselines exhibit severe forgetting due to insufficient coverage of previously learned tasks. In contrast, HTCL-enhanced models (ER+HTCL, DER+HTCL, SER+HTCL) consistently achieve higher accuracy, demonstrating that curvature-aware consolidation compensates for limited rehearsal by providing additional stability.

\begin{figure}[!ht]
    \centering
    \includegraphics[scale=0.1]{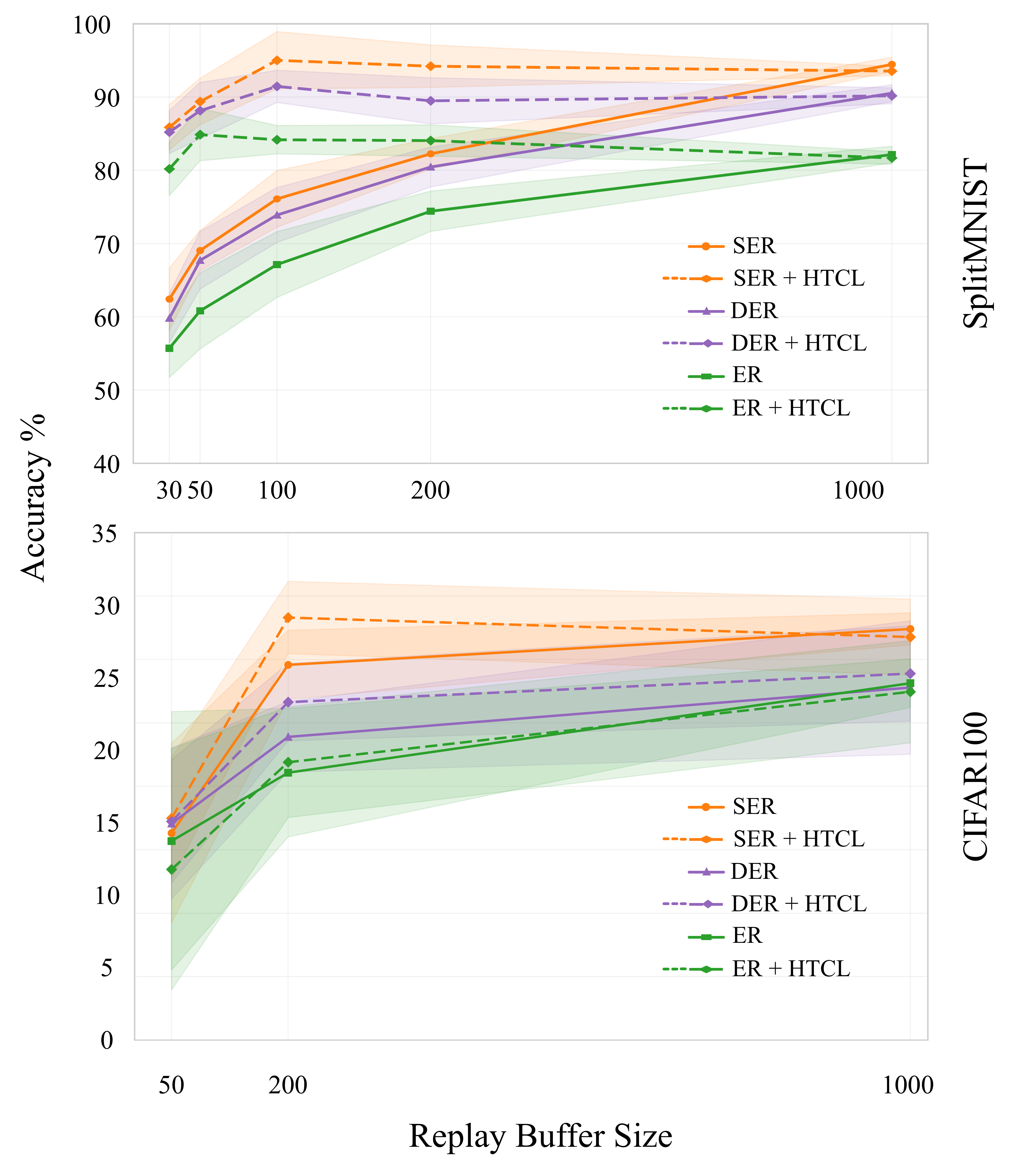}
    \caption{\textbf{Effect of replay buffer size on SplitMNIST and Split CIFAR-100.} 
    We evaluate ER, DER, and SER across varying replay buffer capacities, along with their HTCL-enhanced counterparts. 
    On SplitMNIST, HTCL provides substantial gains under small buffers (30--100 samples), where standalone replay methods suffer severe forgetting due to limited task coverage. As buffer size increases, replay baselines approach joint-training performance and the relative advantage of HTCL diminishes. 
    On Split CIFAR-100, extremely small buffers (e.g., 50 samples) produce highly unstable and low-performing results, reflecting the inadequacy of such memory for a 100-class dataset. Increasing the buffer to 200 leads to sizable improvements, and using 1000 samples further narrows the gap between baseline and HTCL methods. 
    Overall, HTCL offers the greatest benefit in low-memory conditions, providing curvature-aware consolidation that stabilizes learning when replay alone is insufficient.}
    \label{fig:memory}
\end{figure}

As buffer capacity increases, the performance gap between the baselines and their HTCL-augmented counterparts narrows. On SplitMNIST, replay approximates joint training once the buffer reaches moderate size, at which point model capacity becomes the dominant limiting factor and HTCL offers only marginal improvements. On Split CIFAR-100, extremely small buffers (e.g., 50 samples) produce unstable and low-quality performance, reflecting the insufficiency of such memory to represent 100 classes. Increasing the buffer to 200 yields a substantial improvement, and further increasing it to 1000 reduces the benefits of HTCL, mirroring the saturation observed in SplitMNIST.

These results confirm that HTCL is particularly advantageous in low-memory regimes: it acts as a complementary stabilization mechanism that mitigates forgetting when replay alone is inadequate, while remaining lightweight enough to avoid hindering performance when memory is plentiful.

\subsubsection{Relation to Federated Aggregation Methods}
\label{app:federated}

Although HTCL shares a structural similarity with federated learning methods, their objectives and use-cases are very different. Both maintain multiple models and periodically consolidate their updates into a single global model. For completeness and to evaluate these approaches for their ability to consolidate performance across tasks, we compared HTCL against two widely used federated baselines, FedAvg~\citep{mcmahan2017communication} and FedProx~\citep{li2020federated}. 

\textbf{Experimental setup.} In our experiment, we treat each task in the continual learning sequence as analogous to a federated client. For each task, we first train a local model using either SER or DER as the base continual learning method. After training on each task, we update a global model using one of three consolidation strategies. These are HTCL, FedAvg, or FedProx. We then evaluate the final global model on all tasks and measure both accuracy and sensitivity to task ordering.

\textbf{Balanced performance across tasks.} Figure~\ref{fig:fed-comp}(a) shows per-task accuracy on SplitMNIST. FedAvg and FedProx achieve high accuracy on some tasks but at the cost of degraded performance on others. This uneven pattern suggests that simple averaging or proximal regularization does not account for the sequential structure of continual learning. HTCL achieves more balanced performance across all five tasks. 

\textbf{Task-order robustness.} Figure~\ref{fig:fed-comp}(b) compares the standard deviation of accuracy across different task orderings. Both FedAvg and FedProx reduce variance compared to the raw SER and DER baselines, indicating some degree of stabilization. However, HTCL achieves the lowest standard deviation overall. This difference highlights a fundamental distinction between the two settings. Federated methods dampen update magnitudes through averaging or regularization. HTCL instead consolidates knowledge using second-order information that is tailored to sequential, temporal learning.

\begin{figure}[!ht]
    \centering
    \includegraphics[scale=0.12]{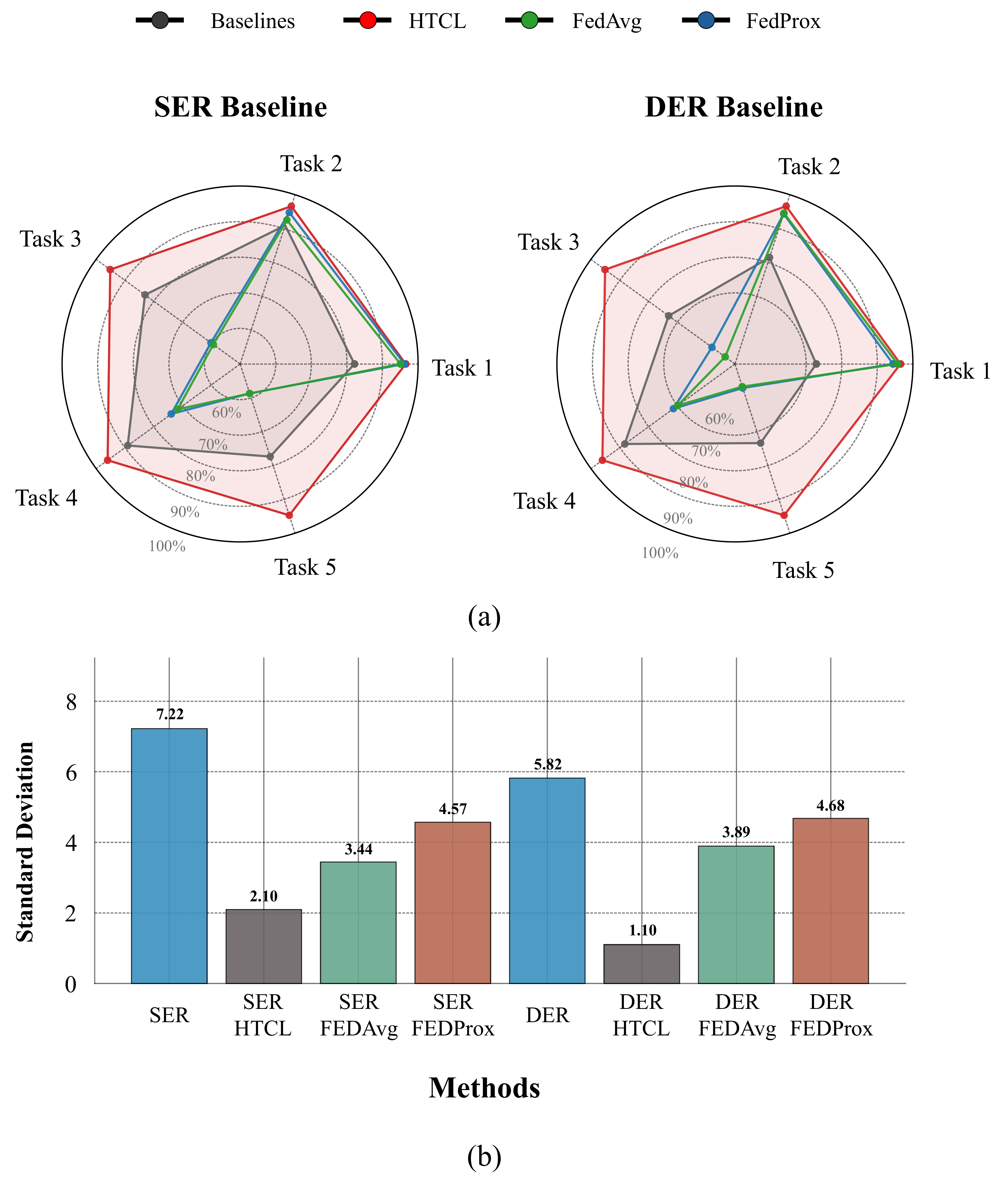}
    \caption{\textbf{Comparison of HTCL with federated learning baselines.} Local models are trained using SER and DER on SplitMNIST, followed by global consolidation using HTCL, FedAvg, or FedProx. SER and DER without consolidation serve as baselines. (a) HTCL achieves consistently balanced accuracy across all tasks, whereas FedAvg and FedProx attain higher accuracy on some tasks at the expense of others. (b) Standard deviation across task permutations shows that HTCL provides the strongest reduction in task-order sensitivity compared to federated baselines.}
    \label{fig:fed-comp}
\end{figure}

%% file: sections/related_work.tex
\subsection{Related Works}\label{app:related_work}

\textbf{Paradigms in CL.}
In recent years, interest in CL has risen, especially from the deep learning research community~\cite{achille2018life, nguyen2017variational, schwarz2018progress, tasnim2024recast}. CL approaches are generally categorized into three paradigms: \emph{regularization-based}, \emph{replay-based}, and \emph{architecture-based} methods \cite{parisi2019continual,nguyen2019toward,schwarz2018progress}.

Regularization-based methods constrain parameter updates to protect prior knowledge. \emph{Elastic Weight Consolidation} (EWC) \cite{kirkpatrick2017overcoming} and \emph{Synaptic Intelligence} (SI) \cite{zenke2017continual} penalize deviation from parameters important to previous tasks. \emph{Online EWC} \cite{chaudhry2018efficient} and \emph{Stable SGD}~\cite{mirzadeh2020dropout} extend these ideas to continuous task streams and smoother optimization.

Replay-based methods alleviate forgetting by explicitly revisiting prior data or learned representations. \emph{Experience Replay} (ER), \emph{Strong ER (SER)}~\cite{zhuo2023continual}, and \emph{Dark ER (DER)}~\cite{buzzega2020dark} maintain small exemplar buffers, while \emph{Generative Replay} techniques \cite{shin2017continual,van2018generative} synthesize pseudo-samples from previous distributions. Exemplar-based systems such as \emph{iCaRL} \cite{rebuffi2017icarl} combine replay with incremental classifiers.

Architecture-based approaches \cite{fernando2017pathnet,masse2018alleviating,li2017learning} mitigate interference by expanding or routing subnetworks to isolate knowledge across tasks. Although effective in preventing forgetting, they typically incur high computational or memory costs and reduced scalability. Despite their progress, these paradigms collectively continue to struggle with the \emph{stability–plasticity dilemma}, particularly under extended task sequences and non-stationary environments \cite{chen2023stability,de2021continual}.

\textbf{Multi-Timescale and Dual-Model Approaches.}
Beyond these standard paradigms, some recent works employ multiple learners operating at different temporal scales. DualNet~\cite{pham2021dualnet}, for example, maintains a fast learner and a slow learner inspired by Complementary Learning Systems. Their fast model is trained using replayed examples, combining both raw buffer samples and latent features from the slow model. The slow learner is also updated through replay. Although this architecture resembles our use of local and hierarchical models, the mechanisms differ substantially. Unlike dual-model systems which rely on standard rehearsal, HTCL repurposes the replay buffer to approximate the curvature of the loss landscape (Hessian) for a regularization step, rather than for blind rehearsal. Consequently, the hierarchical models evolve solely through a Taylor series approximation of the loss, enabling principled multi-timescale integration. This distinction allows HTCL to function as a model-agnostic consolidation layer rather than a replay-driven dual-network system.

\textbf{Task Ordering and Sequence Sensitivity.}
Recent studies demonstrate that task order profoundly affects CL performance \cite{bell2022effect,riemer2018learning,yoon2019scalable}. Models trained on identical tasks but different sequences can yield widely varying accuracies due to path-dependent optimization. While curriculum learning leverages structured ordering to encourage transfer, standard CL frameworks generally assume fixed or random task orders, lacking robustness under arbitrary permutations. Efforts to address this include additive parameter decomposition for order invariance \cite{yoon2019scalable} and order-robust objectives or loss regularizers \cite{dohare2024loss,guo2024continual}. Nonetheless, the factorial complexity of possible permutations makes brute-force optimization infeasible \cite{knoblauch2020optimal}, leaving order sensitivity an open challenge.

\textbf{Relation to Federated Learning.}
Our hierarchical design shares superficial similarity with multi-model frameworks in \emph{Federated Learning} (FL) \cite{mcmahan2017communication,li2020federated}. In FL, distributed clients train local models whose updates are aggregated on a central server to preserve privacy and reduce communication cost. FedAvg~\cite{mcmahan2017communication} aggregates local models by simple parameter averaging, implicitly assuming that all local updates are equally compatible. FedProx~\cite{li2020federated} extends this approach by introducing a proximal term that constrains local updates to remain close to the global model, improving stability under heterogeneous client distributions. FedCurv~\cite{shoham2019overcoming} employs curvature-based regularization to further stabilize aggregation.

However, CL differs fundamentally  from FL. FL addresses spatial heterogeneity across distributed clients, whereas CL handles temporal distributional shifts within a single learner. The curvature-aware consolidation in HTCL is conceptually related to methods like FedCurv, but HTCL integrates information hierarchically over time rather than across parallel clients. \textbf{Detailed experimental comparisons with FL baselines can be found in Appendix~\ref{app:federated}}, confirming that HTCL outperforms both FedAvg and FedProx when applied to sequential task integration and highlighting the importance of consolidation mechanisms tailored to temporal learning.

\textbf{Summary and Distinction.}
In summary, prior methods have mitigated forgetting via replay, regularization, or network expansion, but they remain vulnerable to task-order bias and scalability challenges. Few works explicitly formalize order robustness. In contrast, \emph{HTCL} introduces a principled, multi-level Taylor series framework that unifies local adaptability and hierarchical stability, yielding order-invariant and scalable CL behavior.

%% file: example_paper.bib
@inproceedings{ICLR2025_5565ab68,
 author = {Lewandowski, Alex and Bortkiewicz, Micha\l  and Kumar, Saurabh and Gyorgy, Andras and Schuurmans, Dale and Ostaszewski, Mateusz and C. Machado, Marlos},
 booktitle = {International Conference on Learning Representations},
 editor = {Y. Yue and A. Garg and N. Peng and F. Sha and R. Yu},
 pages = {34504--34530},
 title = {Learning Continually by Spectral Regularization},
 volume = {2025},
 year = {2025}
}

@article{mccallum2000automating,
  title={Automating the construction of internet portals with machine learning},
  author={McCallum, Andrew Kachites and Nigam, Kamal and Rennie, Jason and Seymore, Kristie},
  journal={Information Retrieval},
  volume={3},
  number={2},
  pages={127--163},
  year={2000},
  publisher={Springer}
}

@incollection{lang1995newsweeder,
  title={Newsweeder: Learning to filter netnews},
  author={Lang, Ken},
  booktitle={Machine learning proceedings 1995},
  pages={331--339},
  year={1995},
  publisher={Elsevier}
}

@article{krizhevsky2009learning,
  title={Learning multiple layers of features from tiny images},
  author={Krizhevsky, Alex and Hinton, Geoffrey and others},
  year={2009},
  publisher={Toronto, ON, Canada}
}

@article{lecun2002gradient,
  title={Gradient-based learning applied to document recognition},
  author={LeCun, Yann and Bottou, L{\'e}on and Bengio, Yoshua and Haffner, Patrick},
  journal={Proceedings of the IEEE},
  volume={86},
  number={11},
  pages={2278--2324},
  year={2002},
  publisher={Ieee}
}

@article{shoham2019overcoming,
  title={Overcoming forgetting in federated learning on non-iid data},
  author={Shoham, Neta and Avidor, Tomer and Keren, Aviv and Israel, Nadav and Benditkis, Daniel and Mor-Yosef, Liron and Zeitak, Itai},
  journal={arXiv preprint arXiv:1910.07796},
  year={2019}
}

@inproceedings{mirzadeh2020dropout,
  title={Dropout as an implicit gating mechanism for continual learning},
  author={Mirzadeh, Seyed Iman and Farajtabar, Mehrdad and Ghasemzadeh, Hassan},
  booktitle={Proceedings of the IEEE/CVF conference on computer vision and pattern recognition workshops},
  pages={232--233},
  year={2020}
}

@article{tasnim2024recast,
  title={RECAST: Reparameterized, Compact weight Adaptation for Sequential Tasks},
  author={Tasnim, Nazia and Plummer, Bryan A},
  journal={arXiv preprint arXiv:2411.16870},
  year={2024}
}

@article{li2020federated,
  title={Federated optimization in heterogeneous networks},
  author={Li, Tian and Sahu, Anit Kumar and Zaheer, Manzil and Sanjabi, Maziar and Talwalkar, Ameet and Smith, Virginia},
  journal={Proceedings of Machine learning and systems},
  volume={2},
  pages={429--450},
  year={2020}
}

@inproceedings{mcmahan2017communication,
  title={Communication-efficient learning of deep networks from decentralized data},
  author={McMahan, Brendan and Moore, Eider and Ramage, Daniel and Hampson, Seth and y Arcas, Blaise Aguera},
  booktitle={Artificial intelligence and statistics},
  pages={1273--1282},
  year={2017},
  organization={PMLR}
}

@inproceedings{knoblauch2020optimal,
  title={Optimal continual learning has perfect memory and is np-hard},
  author={Knoblauch, Jeremias and Husain, Hisham and Diethe, Tom},
  booktitle={International Conference on Machine Learning},
  pages={5327--5337},
  year={2020},
  organization={PMLR}
}

@article{buzzega2020dark,
  title={Dark experience for general continual learning: a strong, simple baseline},
  author={Buzzega, Pietro and Boschini, Matteo and Porrello, Angelo and Abati, Davide and Calderara, Simone},
  journal={Advances in neural information processing systems},
  volume={33},
  pages={15920--15930},
  year={2020}
}

@inproceedings{schwarz2018progress,
  title={Progress \& compress: A scalable framework for continual learning},
  author={Schwarz, Jonathan and Czarnecki, Wojciech and Luketina, Jelena and Grabska-Barwinska, Agnieszka and Teh, Yee Whye and Pascanu, Razvan and Hadsell, Raia},
  booktitle={International conference on machine learning},
  pages={4528--4537},
  year={2018},
  organization={PMLR}
}

@article{nguyen2017variational,
  title={Variational continual learning},
  author={Nguyen, Cuong V and Li, Yingzhen and Bui, Thang D and Turner, Richard E},
  journal={arXiv preprint arXiv:1710.10628},
  year={2017}
}

@article{achille2018life,
  title={Life-long disentangled representation learning with cross-domain latent homologies},
  author={Achille, Alessandro and Eccles, Tom and Matthey, Loic and Burgess, Chris and Watters, Nicholas and Lerchner, Alexander and Higgins, Irina},
  journal={Advances in Neural Information Processing Systems},
  volume={31},
  year={2018}
}

@article{nguyen2019toward,
  title={Toward understanding catastrophic forgetting in continual learning},
  author={Nguyen, Cuong V and Achille, Alessandro and Lam, Michael and Hassner, Tal and Mahadevan, Vijay and Soatto, Stefano},
  journal={arXiv preprint arXiv:1908.01091},
  year={2019}
}

@article{chen2023stability,
  title={On the stability-plasticity dilemma in continual meta-learning: Theory and algorithm},
  author={Chen, Qi and Shui, Changjian and Han, Ligong and Marchand, Mario},
  journal={Advances in Neural Information Processing Systems},
  volume={36},
  pages={27414--27468},
  year={2023}
}

@inproceedings{de2021continual,
  title={Continual prototype evolution: Learning online from non-stationary data streams},
  author={De Lange, Matthias and Tuytelaars, Tinne},
  booktitle={Proceedings of the IEEE/CVF international conference on computer vision},
  pages={8250--8259},
  year={2021}
}

@article{bell2022effect,
  title={The effect of task ordering in continual learning},
  author={Bell, Samuel J and Lawrence, Neil D},
  journal={arXiv preprint arXiv:2205.13323},
  year={2022}
}

@article{chaudhry2018efficient,
  title={Efficient lifelong learning with a-gem},
  author={Chaudhry, Arslan and Ranzato, Marc'Aurelio and Rohrbach, Marcus and Elhoseiny, Mohamed},
  journal={arXiv preprint arXiv:1812.00420},
  year={2018}
}

@inproceedings{rebuffi2017icarl,
  title={icarl: Incremental classifier and representation learning},
  author={Rebuffi, Sylvestre-Alvise and Kolesnikov, Alexander and Sperl, Georg and Lampert, Christoph H},
  booktitle={Proceedings of the IEEE conference on Computer Vision and Pattern Recognition},
  pages={2001--2010},
  year={2017}
}

@article{lopez2017gradient,
  title={Gradient episodic memory for continual learning},
  author={Lopez-Paz, David and Ranzato, Marc'Aurelio},
  journal={Advances in neural information processing systems},
  volume={30},
  year={2017}
}

@article{kirkpatrick2017overcoming,
  title={Overcoming catastrophic forgetting in neural networks},
  author={Kirkpatrick, James and Pascanu, Razvan and Rabinowitz, Neil and Veness, Joel and Desjardins, Guillaume and Rusu, Andrei A and Milan, Kieran and Quan, John and Ramalho, Tiago and Grabska-Barwinska, Agnieszka and others},
  journal={Proceedings of the national academy of sciences},
  volume={114},
  number={13},
  pages={3521--3526},
  year={2017},
  publisher={National Acad Sciences}
}

@article{yoon2019scalable,
  title={Scalable and order-robust continual learning with additive parameter decomposition},
  author={Yoon, Jaehong and Kim, Saehoon and Yang, Eunho and Hwang, Sung Ju},
  journal={arXiv preprint arXiv:1902.09432},
  year={2019}
}

@article{li2017learning,
  title={Learning without forgetting},
  author={Li, Zhizhong and Hoiem, Derek},
  journal={IEEE transactions on pattern analysis and machine intelligence},
  volume={40},
  number={12},
  pages={2935--2947},
  year={2017},
  publisher={IEEE}
}

@article{riemer2018learning,
  title={Learning to learn without forgetting by maximizing transfer and minimizing interference},
  author={Riemer, Matthew and Cases, Ignacio and Ajemian, Robert and Liu, Miao and Rish, Irina and Tu, Yuhai and Tesauro, Gerald},
  journal={arXiv preprint arXiv:1810.11910},
  year={2018}
}

@article{raghavan2021formalizing,
  title={Formalizing the generalization-forgetting trade-off in continual learning},
  author={Raghavan, Krishnan and Balaprakash, Prasanna},
  journal={Advances in Neural Information Processing Systems},
  volume={34},
  pages={17284--17297},
  year={2021}
}

@article{parisi2019continual,
  title={Continual lifelong learning with neural networks: A review},
  author={Parisi, German I and Kemker, Ronald and Part, Jose L and Kanan, Christopher and Wermter, Stefan},
  journal={Neural networks},
  volume={113},
  pages={54--71},
  year={2019},
  publisher={Elsevier}
}

@article{dohare2024loss,
  title={Loss of plasticity in deep continual learning},
  author={Dohare, Shibhansh and Hernandez-Garcia, J Fernando and Lan, Qingfeng and Rahman, Parash and Mahmood, A Rupam and Sutton, Richard S},
  journal={Nature},
  volume={632},
  number={8026},
  pages={768--774},
  year={2024},
  publisher={Nature Publishing Group UK London}
}

@article{pham2021dualnet,
  title={Dualnet: Continual learning, fast and slow},
  author={Pham, Quang and Liu, Chenghao and Hoi, Steven},
  journal={Advances in Neural Information Processing Systems},
  volume={34},
  pages={16131--16144},
  year={2021}
}

@article{zhuo2023continual,
  title={Continual learning with strong experience replay},
  author={Zhuo, Tao and Cheng, Zhiyong and Gao, Zan and Fan, Hehe and Kankanhalli, Mohan},
  journal={arXiv preprint arXiv:2305.13622},
  year={2023}
}

@inproceedings{guo2024continual,
  title={Continual learning in an open and dynamic world},
  author={Guo, Yunhui},
  booktitle={Proceedings of the AAAI Conference on Artificial Intelligence},
  volume={38},
  pages={22666--22666},
  year={2024}
}

@article{shin2017continual,
  title={Continual learning with deep generative replay},
  author={Shin, Hanul and Lee, Jung Kwon and Kim, Jaehong and Kim, Jiwon},
  journal={Advances in neural information processing systems},
  volume={30},
  year={2017}
}

@inproceedings{zenke2017continual,
  title={Continual learning through synaptic intelligence},
  author={Zenke, Friedemann and Poole, Ben and Ganguli, Surya},
  booktitle={Proceedings of the 34th International Conference on Machine Learning},
  pages={3987--3995},
  year={2017}
}

@article{masse2018alleviating,
  title={Alleviating catastrophic forgetting using context-dependent gating and synaptic stabilization},
  author={Masse, Nicolas Y and Grant, Gregory D and Freedman, David J},
  journal={Proceedings of the National Academy of Sciences},
  volume={115},
  number={44},
  pages={E10467--E10475},
  year={2018}
}

@inproceedings{fernando2017pathnet,
  title={PathNet: Evolution channels gradient descent in super neural networks},
  author={Fernando, Chrisantha and Banarse, Dylan and Blundell, Charles and Zwols, Yori and Ha, David and Rusu, Andrei A and Pritzel, Alexander and Wierstra, Daan},
  booktitle={Proceedings of the Genetic and Evolutionary Computation Conference},
  pages={169--178},
  year={2017}
}

@article{van2018generative,
  title={Generative replay with feedback connections as a general strategy for continual learning},
  author={van de Ven, Gido M and Tolias, Andreas S},
  journal={arXiv preprint arXiv:1809.10635},
  year={2018}
}
